\documentclass[10pt,twocolumn,letterpaper]{article}

\usepackage[pagenumbers]{cvpr} 
\usepackage{amsthm}

\newtheorem{lemma}{Lemma}
\newtheorem{theorem}{Theorem}

\newtheorem{definition}{Definition}

\usepackage{graphicx} 
\usepackage{subfloat}
\usepackage{multirow}
\usepackage{bm}
\usepackage{booktabs}
\usepackage{array}
\usepackage{makecell}
\usepackage{ragged2e}
\usepackage{enumitem}
\usepackage{amsmath}
\usepackage{amssymb}

\usepackage{algorithm}
\usepackage{algorithmic}

\definecolor{cvprblue}{rgb}{0.21,0.49,0.74}
\usepackage[pagebackref,breaklinks,colorlinks,allcolors=cvprblue]{hyperref}

\def\paperID{12117} 
\def\confName{CVPR}
\def\confYear{2025}

\title{Weakly Supervised Contrastive Adversarial Training for Learning Robust Features from Semi-supervised Data}

\author{Lilin Zhang, Chengpei Wu, Ning Yang\thanks{Corresponding author.}\\
School of Computer Science, Sichuan University, Chengdu, China\\
{\tt\small zhanglilin@stu.scu.edu.cn, wuchengpei@stu.scu.edu.cn, yangning@scu.edu.cn}
}

\begin{document}
\maketitle
\begin{abstract}
Existing adversarial training (AT) methods often suffer from incomplete perturbation, meaning that not all non-robust features are perturbed when generating adversarial examples (AEs). This results in residual correlations between non-robust features and labels, leading to suboptimal learning of robust features. However, achieving complete perturbation—perturbing as many non-robust features as possible—is challenging due to the difficulty in distinguishing robust and non-robust features and the sparsity of labeled data. To address these challenges, we propose a novel approach called Weakly Supervised Contrastive Adversarial Training (WSCAT). WSCAT ensures complete perturbation for improved learning of robust features by disrupting correlations between non-robust features and labels through complete AE generation over partially labeled data, grounded in information theory. Extensive theoretical analysis and comprehensive experiments on widely adopted benchmarks validate the superiority of WSCAT. Our code is available at \href{https://github.com/zhang-lilin/WSCAT}{https://github.com/zhang-lilin/WSCAT}. 
\end{abstract}
\section{Introduction}
\label{sec:intro}
Deep neural networks (DNNs) have achieved remarkable success, but research has revealed their vulnerability to adversarial examples (AEs), which are generated by subtly perturbing natural samples \cite{biggio2013evasion,szegedy2013intriguing,ilyas2019adversarial,miller2020adversarial}. The threat posed by AEs has spurred ongoing efforts to enhance adversarial robustness—i.e., a model’s accuracy on AEs—with adversarial training (AT) being widely recognized as the most promising defense strategy \cite{bai2021recent, zhao2022adversarial}. AT employs a min-max game, where attacks maximize adversarial loss by generating AEs, while defenses minimize adversarial loss by adjusting the classifier, ultimately training a robust model capable of withstanding worst-case attacks.


\begin{figure}[t]
\centering
\subfloat[Incomplete perturbation]{\includegraphics[width=.45\columnwidth]{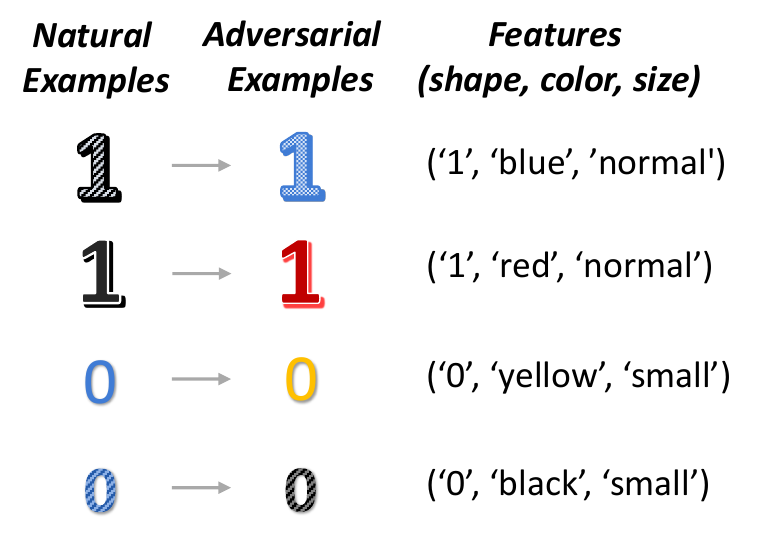} \label{fig-existing-ae}} 
\subfloat[Complete perturbation]{\includegraphics[width=.45\columnwidth]{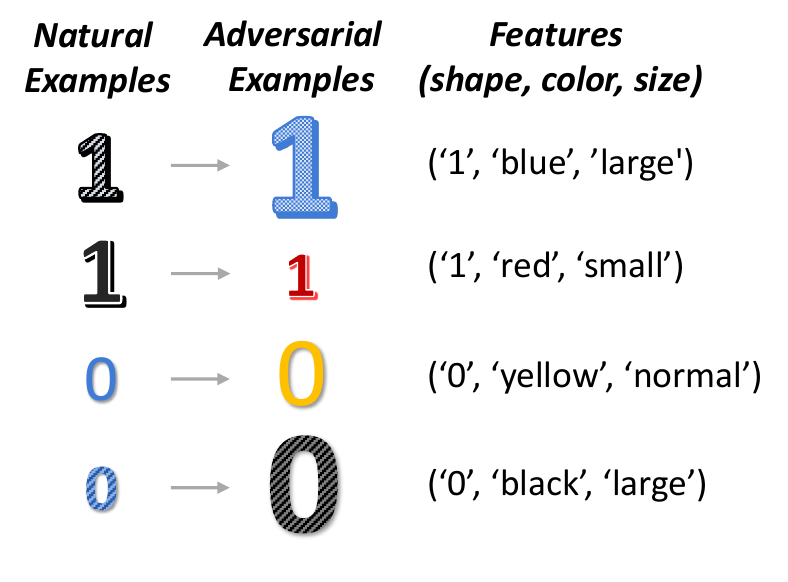} \label{fig-ours-ae}}
\caption{Illustration of non-robust features and their perturbations. For a digit image, color and size are the non-robust features correlating to the labels, while shape is the robust feature causing the labels. 
}
\label{Fig-motivation}
\end{figure}

Existing AT methods can be categorized into two groups based on how AEs are generated: point-wise AT \cite{madry2017towards, zhang2019theoretically, wang2023better, li2024focus, wang2024revisiting}, which perturbs individual natural samples independently, and distributional AT \cite{staib2017distributionally, zhang2024provable}, which samples AEs from a perturbed distribution of natural samples. Despite their success, these AT methods still face challenges in \textit{robust features learning}. Recent studies suggest that the existence of AEs stems from the presence of non-robust features in natural data \cite{ilyas2019adversarial}, which may exhibit statistical correlations with labels and can be manipulated by AEs. In contrast, robust features have a strong causal relationship with labels and remain unchanged under perturbation constraints.

However, current AT methods often blindly search for adversarial perturbations within an $\epsilon$-ball defined by a norm constraint \cite{cai2024and}, leading to \textit{incomplete perturbation}, where not all non-robust features are perturbed. The unperturbed non-robust features retain their correlation with labels, reducing adversarial robustness. Ideally, AEs should eliminate all correlations between non-robust features and labels through complete perturbation, ensuring that the model learns only robust features. For example, in \cref{Fig-motivation}, color and size are the non-robust features of digit images, while shape is the robust feature. In \cref{fig-existing-ae}, only the color of a digit image is modified, while size remains unchanged, allowing size to still correlate with labels. In contrast, \cref{fig-ours-ae} illustrates complete perturbation, where both color and size are altered, forcing the model to rely solely on shape as a robust feature.

Achieving complete perturbation in real-world scenarios is challenging due to two key obstacles. First, robust and non-robust features are not explicitly identifiable, as there are no supervisory signals to indicate their locations. Second, the sparsity of labeled data exacerbates this problem. To address these challenges, we propose Weakly Supervised Contrastive Adversarial Training (WSCAT), which leverages partially labeled data to guide a target classifier in learning robust features.

Specifically, we first notice that the mutual information between the embeddings of natural and adversarial data measures the amount of correlations to be preserved, minimizing which can block the correlations between non-robust features and labels. Inspired by this observation, we propose an information-theoretic complete AE generation for searching complete perturbations by introducing an additional mutual information constraint to the optimization of AE. Further, we replace the intractable mutual information with a weakly supervised dynamic loss to fulfill complete AE generation. The weakly supervised dynamic loss incorporates the predictions from current target classifier into InfoNCE loss, which gives a lower bound of mutual information \cite{poole2019variational}, to dynamically reflect the correlations between non-robust features and labels without the need of labels. The insight is that our approach does not explicitly seeking to identify robust or non-robust features. Instead, WSCAT blocks the correlations by generating complete AEs, ultimately achieving the goal of distilling non-robust features out of the sample embeddings. Our main contributions can be summarized as follows: 
\begin{itemize}
	\item We identify the issue of incomplete perturbation, which hinders the learning of robust features and hence leads to suboptimal adversarial robustness of target classifiers.
	
	\item We propose a novel solution called Weakly Supervised Contrastive Adversarial Training (WSCAT). WSCAT fulfills the learning of robust features on semi-supervised dataset by blocking the correlations between non-robust features and labels, via complete AE generation over partially labeled data, in a weakly supervised manner.
	
	\item  The solid theoretical analysis and the extensive experiments conducted on widely adopted benchmarks verify the superiority of WSCAT.
\end{itemize}

\section{Preliminary}
\label{sec:pre}

\subsection{Robust Feature}
The seminar work \cite{ilyas2019adversarial} demonstrates that adversarial vulnerability is a consequence of non-robust features in the data. These non-robust features are useful and highly predictive for standard generalization, yet easily broken by adversarial perturbations. In contrast, robust features can provide robust predictions even under adversarial perturbations, which are generally useful for the predictions of both natural samples and AEs. 

Let $\mathcal{X}$ be the set of samples, $\mathcal{Z} \subseteq \mathbb{R}^{k}$ be the $k$-dimensional embedding space, and $\mathcal{Y}$ be the set of labels. Let $C = g \circ f : \mathcal{X} \to \mathcal{Y}$ be the target classifier, where $f: \mathcal{X} \to \mathcal{Z}$ is the feature encoder and $g: \mathcal{Z} \to \mathcal{Y}$ is the decoder. Let $X$, $X^{\prime}$ and $Y$ be the random variables representing natural samples, AEs and labels, respectively. 
According to the idea of the existing works \cite{ilyas2019adversarial, kim2021distilling}, the concept of robust feature can be defined as follow: 
\begin{definition} [Robust Feature] Given a distance metric $d(\cdot , \cdot)$, loss function $l (\cdot , \cdot)$ and a perturbation budget $\epsilon$, the feature $f$ is $\rho_l$-$\gamma_d$-robust if $\rho_l = \inf_{g} \mathrm{E}_{ P_{X,Y}} [ l( g \circ f(X), Y ) ]$, and $\gamma_d = \mathrm{E}_{ P_{X}} [ \sup_{X^\prime \in \mathcal{B}_{\epsilon}(X)} d( f(X^\prime) , f(X) ) ]$, where $P_{X, Y}$ is the joint distribution of $X$ and $Y$, $P_{X}$ is the marginal distribution of $X$, and $\mathcal{B}_{\epsilon}(X) = \{ X^\prime $$: \Vert X^\prime - X \Vert_{\infty} \le \epsilon \}$ is the $\epsilon$-ball centered in $X$.
\label{Def_robust_feature}
\end{definition}
The idea of the above definition is that a robust feature $f$ is supposed to be predictive (resulting in smaller $\rho_l$) and invariant (resulting in smaller distance $\gamma_d$). \cref{Def_robust_feature} also tells us that robust features do not dramatically change under imperceptible perturbation constraint offered by the $\epsilon$-ball. Therefore, adversarial attacks can only rely on perturbing the non-robust features (which is also predictive and thus captured by the target model) to achieve the purpose of changing the model predictions. 

\subsection{Adversarial Training}
The traditional AT can be conceptually formulated as the following min-max game \cite{szegedy2013intriguing,bai2021recent,zhao2022adversarial}:
\begin{equation}
\min_{C} \mathrm{E}_{(X,Y) \sim P_{X,Y}} \big[ \max_{X^\prime \in \mathcal{B}_{\epsilon} (X)} l (X^\prime, X, Y) \big] , 
\label{Eq-AT}
\end{equation}
where $P_{X, Y}$ is the joint distribution of natural sample $X$ and its label $Y$, and $l (X^\prime, X, Y )$ is the loss on both natural and adversarial data of target classifier $C$. The loss $l (X^\prime, X, Y )$ can be implemented by any qualified forms, e.g., $\mathrm{CE} \big(C(X), Y\big) + \lambda \mathrm{CE} \big(C(X^\prime), Y\big)$ \cite{madry2017towards} and $\mathrm{CE} \big(C(X), Y\big) + \lambda \mathrm{KL} \big(C(X) \Vert C(X^\prime)\big)$ \cite{zhang2019theoretically}, where $\mathrm{CE}$ represents cross-entropy loss, $\mathrm{KL}$ represents Kullback–Leibler divergence, and $\lambda$ controls the trade-off between standard generalization and adversarial robustness. 


\section{Methodology}

\subsection{Information-theoretic Motivation}
\label{sec:moti}
The inner maximization in \cref{Eq-AT}, which plays the role of generating the AEs within the $\epsilon$-ball $\mathcal{B}_\epsilon(X)$, does not consider the amount of perturbed non-robust features, leading to incomplete adversarial perturbation and suboptimal robustness.  We argue that AEs should be generated via complete perturbation, ensuring that all non-robust features are disrupted so that their correlation with labels is eliminated.

In order to study AEs at feature level, we reformulate the AE generation process process into a distributional form:  
\begin{equation}
\begin{aligned}
\max_{P_{X^\prime}} \text{\quad} & \mathrm{E}_{(X,Y) \sim P_{X, Y} } \text{ }  l(X^\prime, X, Y)  \\ 
\text{s.t.\quad}  &\mathcal{W}_{\infty} (P_{X^\prime}, P_X) \le \epsilon, 
\end{aligned}
\label{Eq:ae}
\end{equation}
where $P_{X^\prime}$ represents the AE distribution, and $\mathcal{W}_{\infty} $ is the $p$-Wasserstein distance $\mathcal{W}_{p} $ \cite{villani2009optimal} for $p=\infty$, ensuring that AEs remain within the $\epsilon$-ball of natural samples. 
Specifically, $\mathcal{W}_{p} (P_{X^\prime}, P_X) = \inf_{\gamma \in \Gamma(P_{X^\prime}, P_X) } \{ \int \Vert X^\prime - X \Vert_{\infty}^p d \gamma \}^{1/p} $, where $\Gamma(P_{X^\prime}, P_X)$ is the set of couplings whose marginals are $P_{X^\prime}$ and $P_{X}$. Therefore, for $p=\infty$, no point between $P_X$ and $P_{X^\prime}$ can move more than $\epsilon$ from $X$ to $X^\prime$, i.e., $\forall X \sim P_{X}$, its corresponding AE $X^\prime \in \mathcal{B}_\epsilon(X)$ \cite{staib2017distributionally}. 

Since non-robust features lack explicit supervisory signals, we introduce a mutual information constraint to enforce complete perturbation. Specifically, let $Z^\prime = f(X^\prime)$ and $Z = f(X)$ denote the feature embeddings of AEs and natural samples, respectively. Assuming $Z$ consists of robust features $R$ and $K$ non-robust features $Z_{1:K} = \{Z_1, Z_2, ... ,Z_{K} \}$, we decompose the mutual information between features of natural and adversarial data $\mathrm{MI} (Z^\prime; Z)$ according to the chain rule of mutual information \cite{alajaji2018introduction}:
\begin{equation}
\begin{aligned}
\mathrm{MI} (Z^\prime; Z) = \mathrm{MI} (Z^\prime; R) + \sum_{k=1}^K \mathrm{MI} (Z^\prime; Z_k \vert R, Z_{1:k-1}). 
\end{aligned}
\end{equation} 
Due to the property of robust features (see \cref{Def_robust_feature}), robust features remain unchanged under adversarial perturbations, and thus $\mathrm{MI} (R; Z^\prime)$ is a constant equal to the entropy $\mathrm{H}(R)$. Minimizing $\mathrm{MI} (Z_k; Z^\prime \vert R, Z_{1:k-1})$ ensures that non-robust features are perturbed in a distributional sense. Thus, we define our information-theoretic complete AE generation as:
\begin{equation}
\begin{aligned}
\max_{P_{X^\prime}} \text{\quad} & \mathrm{E}_{(X,Y) \sim P_{X,Y} } \text{ }  l(X^\prime, X, Y) \\ 
\text{s.t.\quad}  & \mathcal{W}_{\infty} (P_{X^\prime}, P_X) \le \epsilon, \mathrm{MI}( Z^\prime; Z) = \mathrm{MI} (Z^\prime; R). 
\end{aligned}
\label{Eq-constrain}
\end{equation}
where $\mathrm{MI}( Z^\prime; Z) = \mathrm{MI} (Z^\prime; R)$ is an additional perturbation constraint beyond the perturbation budget $\epsilon$. However, it is difficult to directly solve the optimization problem defined by \cref{Eq-constrain}. We hence instead consider a Lagrangian relaxation with a fixed penalty parameter $ \beta \ge 0$, which results in the following feasible dual problem: 
\begin{equation}
\max_{P_{X^\prime} : \mathcal{W}_{\infty} (P_{X^\prime}, P_X) \le \epsilon} \big\{ \mathrm{E}_{P_{X,Y}} \text{ } l(X^\prime, X, Y)  - \beta \mathrm{MI}( Z^\prime; Z) \big\}.
\label{Eq-dual}
\end{equation}
Obviously, the optimization of \cref{Eq-dual} requires the minimization of $\mathrm{MI}( Z^\prime; Z)$, and the greater $\beta$, the less reserved mutual information. \cref{Fig-constrain} gives an illustration of our information-theoretic complete AE generation. As shown in \cref{fig-exact-beta}, the optimal selection of $\beta$, i.e., $\beta^*$ (corresponding to $\mathrm{MI}(Z^\prime; R)$ in duality), can force all the non-robust features to be perturbed and make $\mathrm{MI} (Z^\prime; Z)$ exactly cover the robust features. In contrast, excessively small $\beta < \beta^*$ leads to incomplete perturbation (\cref{fig-large-beta}), while excessively large $\beta >  \beta^*$ causes the loss of robust features (\cref{fig-small-beta}). 
\begin{figure}[t]
\centering
\subfloat[$\beta $$<$$ \beta^*$]{\includegraphics[scale=0.39]{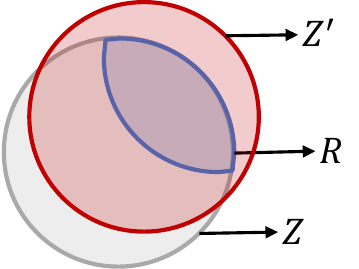} \label{fig-large-beta}} 
\subfloat[$\beta $$=$$ \beta^*$]{\includegraphics[scale=0.37]{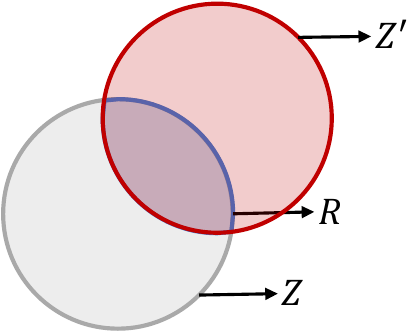} \label{fig-exact-beta}}
\subfloat[$\beta $$>$$ \beta^*$]{\includegraphics[scale=0.37]{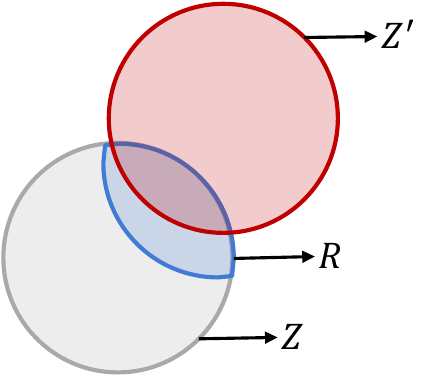} \label{fig-small-beta}}
\caption{Illustration of Information-theoretic Complete AE Generation, where blue, grey, and red areas represent the information contained in the robust features $R$, the features $Z$ of natural samples, and the features $Z^\prime$ of AEs, respectively, and the overlapped area represents mutual information $\mathrm{MI} (Z, Z^\prime)$ contained in the unperturbed features. 
}
\label{Fig-constrain}
\end{figure}

Actually, $Z$ consists of robust and non-robust features captured by the target classifier $C$ for prediction, which implicitly reflects the correlations between (both robust and non-robust) features and labels captured by $C$ by pointing out which features are encoded by $C$. Therefore, $\mathrm{MI}( Z^\prime; Z)$ measures the amount of correlations preserved by $C$, minimizing which will block some correlations. Furthermore, due to the invariant property of robust features, the blocked correlations are the ones between non-robust features and labels, rather than the ones between robust features and labels. Therefore, the additional mutual information constraint will enforce more diverse non-robust features to be perturbed and results in complete perturbation.

\subsection{WSCAT}

\subsubsection{Weakly Supervised Dynamic Loss}
Direct optimization of mutual information is computationally challenging. Let $\mathcal{D} = \mathcal{D}_l \cup \mathcal{D}_u$ be a semi-supervised dataset consists of a labeled dataset $\mathcal{D}_l$ and a unlabeled dataset $\mathcal{D}_u$, and $z$, $z^\prime$ denote the embeddings of $x$ and its AE $x^\prime$, i.e., $z = f(x)$, $z^\prime = f(x^\prime)$. One promising approach to approximate $\mathrm{MI} (Z^\prime; Z)$ is using InfoNCE loss $l_\mathrm{nce} (z^\prime, z)$, which provides a lower bound on it:
\begin{equation}
\begin{aligned}
l_\mathrm{nce} (z^\prime, z) = 
- \log{ \frac{e^{ s(z^\prime, z)} }{ \sum_{x_n \in \mathcal{D} } e^{ s(z^\prime, z_n)} } } , 
\end{aligned}
\label{Eq-NCE}
\end{equation}
where $z_n = f(x_n)$ is the embedding of $x_n$ and $s(\cdot, \cdot)$ is a similarity function. 

However, minimizing this lower bound does not necessarily result in $\mathrm{MI} (Z^\prime; Z)$ being minimized since the former is not a sufficient condition for the latter, which does not guarantee effective blocking of non-robust feature correlations. To address this, we further incorporate predictions from the target classifier $C$, which is so-called weak supervisions, into InfoNCE loss, defining a weakly supervised dynamic loss:
\begin{equation}
\begin{aligned}
l_\mathrm{con} (z^\prime, z) = 
- \frac{1}{ \vert \mathcal{N}^+_x \vert } \sum_{x_p \in \mathcal{N}^+_x} \log{ \frac{ e^{ s(z^\prime, z_p)} }{ \sum_{x_n \in \mathcal{D} } e^{ s(z^\prime, z_n)} } }, 
\end{aligned}
\label{Eq-CON}
\end{equation}
where the positive sample set $\mathcal{N}^+_x$ consists of samples that the current classifier $C$ predicts to belong to the same class as $x$, i.e., $\mathcal{N}^+_x = \{ x_p | x_p \in \mathcal{D} \land \text{argmax}_y C(x_p) =  \text{argmax}_y C(x) \}$. If another sample $x_1$ receives the same prediction as $x$, it shares the same positive sample set, i.e., $\mathcal{N}^+_{x_1} = \mathcal{N}^+_x$. This positive sample set effectively represents the current classifier $C$'s understanding of the label $y = \text{argmax}_y C(x) = \text{argmax}_y g(z)$, indicating which features are associated with this label. Consequently, it reflects the correlations captured by $C$ and facilitates the suppression of correlations with non-robust features. Moreover, as $C$ is continuously updated, the positive sample sets evolve over the course of training, allowing the captured correlations to be dynamically adjusted. This dynamic selection strategy for positive samples relies solely on weak supervision, making it independent of explicit labels and thus capable of leveraging both labeled and unlabeled data.

\subsubsection{Complete AE Generation}
Replacing the mutual information term in \cref{Eq-dual} with the weakly supervised dynamic loss, the complete AE generation process in WSCAT for a given sample $x$ from the semi-supervised dataset $\mathcal{D}$ is formulated as follows:
\begin{equation}
\max_{x^\prime \in \mathcal{B}_{\epsilon} (x)}  \big\{ \mathrm{KL} \big(C(x) \Vert C(x^\prime)\big) + \beta l_\mathrm{con} (z^\prime, z) \big\}, 
\label{Eq-AE}
\end{equation}
where $\mathrm{KL} \big(C(x) \Vert C(x^\prime)\big)$ ensures that the AE $x^\prime$ induces a significant shift in the output of the original sample $x$ (i.e., effectively perturbing non-robust features), while $ l_\mathrm{con} (z^\prime, z)$ promotes a more comprehensive perturbation at a distributional level (i.e., perturbing as many non-robust features as possible). \cref{Eq-AE} thus facilitates the generation of complete AEs, effectively disrupting the correlations between non-robust features and labels captured by the target classifier $C$ on the semi-supervised dataset $\mathcal{D}$. A more detailed analysis of this process is provided later.

As previously discussed in \cref{sec:moti}, selecting an appropriate value for $\beta$ is crucial for ensuring effective complete AE generation, and thus we design a methodology for tuning $\beta$. Since robust features contribute to accurate predictions for both natural and adversarial data, we leverage the harmonic mean—an approach sensitive to lower values—of the natural and robust accuracies on the validation set. This metric serves as an implicit indicator of the robustness of the features learned by target classifier $C$ and is used to determine the optimal value of $\beta$.

%

\subsubsection{Total Optimization}
By generating complete AEs for training, WSCAT's final optimization objective is defined as a min-max game: 
\begin{equation}
\begin{aligned}
& \min_{C=g\circ f} \Big\{ \frac{1}{\vert \mathcal{D} \vert} \sum_{x \in \mathcal{D^*}} \mathrm{CE} (C(x), y) \\
& + \frac{ \lambda}{\vert \mathcal{D} \vert} \sum_{x \in \mathcal{D}}  \max_{x^\prime \in \mathcal{B}_{\epsilon} (x)}  \big\{ \mathrm{KL} \big(C(x) \Vert C(x^\prime)\big) + \beta l_\mathrm{con} (z^\prime, z) \big\} \Big\}, 
\end{aligned}
\label{Eq-WSCAT}
\end{equation}
where $\lambda$ controls the weight of adversarial loss, and $\mathcal{D}^*$ is the augmented semi-supervised dataset. Specifically, we follow \cite{carmon2019unlabeled} to generate pseudo-labels for unlabeled data in a self-training manner by Mean Teacher (MT) algorithm \cite{tarvainen2017mean}, i.e., $\mathcal{D}^* = \mathcal{D}_l \cup \mathcal{D}^*_u$, where $\mathcal{D}_u^* = \{(x_u, y_u) | x_u \in \mathcal{D}_u, y_u = \text{argmax}_{y \in \mathcal{Y}} \text{ } C^* (x_u) \}$ and $C^*$ is the classifier trained by MT. 

In \cref{Eq-WSCAT}, the inner maximization fulfills complete AE generation, and the outer minimization performs the adversarial defense. Particularly, once the complete AEs are generated, the outer minimization will enforce the target model $C$ to capture the predictive features by the first line in \cref{Eq-WSCAT}, and guide $C$ to encode and only encode the robust features for prediction by the second line, which is because the features shared between natural samples and AEs are only robust features due to the complete perturbation offered by the outer AE generation.

\subsection{Theoretical Analysis}

\subsubsection{Correlation Blocking}
In this section, we analyze how the complete AE generation process in WSCAT effectively disrupts the correlations between non-robust features and labels through an additional maximization of the weakly supervised dynamic loss $l_\mathrm{con} (z^\prime, z)$. By design, AT enforces the prediction of an AE $x^\prime$ to remain consistent with that of its corresponding natural sample $x$, i.e.,$\text{argmax}_y g(z^\prime) = \text{argmax}_y g(z)$. Consequently, this consistency ensures that the differences between $z^\prime$ and $z$ no longer retain correlations with the label $y=\text{argmax}_y g(z)$, effectively eliminating their influence. Thus, our focus shifts to analyzing the differences between $z^\prime$ and $z$ that arise due to the maximization of $l_\mathrm{con} (z^\prime, z)$, which leads to the following theorem:
\begin{theorem}
$\max_{x^\prime \in \mathcal{B}_\epsilon (x)} l_\mathrm{con} (z^\prime, z) $ $\approx $ $\max_{x^\prime \in \mathcal{B}_\epsilon (x)} \frac{1}{ \vert \mathcal{N}^+_x \vert }$ $ \sum_{x_p \in \mathcal{N}^+_x}  \max \big\{ 0, \big\{ s(z^\prime, z_n) - s(z^\prime, z_p) \big\}_{x_n \in \mathcal{D}, x_n \ne x_p} \big\}$. 
\label{Th-CB}
\end{theorem}
\cref{Th-CB} demonstrates that maximizing $l_\mathrm{con} (z^\prime, z)$ leads to an adversarial example whose features (i.e., embedding $z^\prime$) become increasingly dissimilar to those of the samples in $\mathcal{N}^+_x$, which represents the features that the target classifier $C$ associates with the label $y=\text{argmax}_y g(z)$. By maximizing $l_\mathrm{con} (z^\prime, z)$, all non-robust features in $z^\prime$ are encouraged to deviate as much as possible from those typically classified as $y=\text{argmax}_y g(z)$, which results in the aforementioned thoroughly blocking. 
 

\subsubsection{Robust Features Learning}
Now we justify WSCAT's ability to learn robust-features satisfying \cref{Def_robust_feature}. Particularly, we will show that the features learned by WSCAT is $\rho_{l_\mathrm{nat}}$-$\gamma_{\Delta}$-robust w.r.t. distance metric $\Delta(z^\prime, z) = \vert l_\mathrm{con}(z^\prime, z) -  l_\mathrm{con}(z, z) \vert$ and loss function $l_\mathrm{nat} (x,y) = \mathrm{CE} (C(x), y) + \lambda \beta  l_\mathrm{con}(z, z)$. For this purpose, we first equivalently reformulate the min-max optimization objective of WSCAT (\cref{Eq-WSCAT}) as: 
\begin{equation}
\begin{aligned}
 \min_{C=g\circ f} \Big\{ & \underbrace{ \frac{1}{\vert \mathcal{D} \vert} \sum_{(x, y)  \in \mathcal{D}^*} l_\mathrm{nat} (x, y) }_{A_1} \\
& + \lambda \underbrace{ \frac{1}{\vert \mathcal{D} \vert} \sum_{x  \in \mathcal{D}} \max_{x^\prime \in \mathcal{B}_{\epsilon} (x)} l_\mathrm{adv} (x^\prime, x) }_{A_2} \Big\},
\end{aligned}
\label{Eq-min-max}
\end{equation}
where $l_\mathrm{adv} (x^\prime, x) = \mathrm{KL} (C(x) \Vert C(x^\prime)) + \beta ( l_\mathrm{con}(z^\prime, z) - l_\mathrm{con}(z, z) ) $. Then, we evaluate the robustness of the features learned by the target classifier $C$, which is optimized according to \cref{Eq-min-max}. This robustness is formally quantified in the following theorem:
\begin{theorem}
Let $\Delta_{\mathrm{m}}$ be the supremum of distance $\Delta (\cdot, \cdot)$ over $\{ (z^\prime, z) \vert x $$ \sim $$ P_{X} \land x^\prime$$ \in $$\mathcal{B}_\epsilon (x) \land z $$=$$ f(x) \land z^\prime $$=$$ f(x^\prime) \}$, $l_{\mathrm{m}}$ be the supremum of loss $l_\mathrm{nat} (\cdot, \cdot)$ over $ P_{X, Y}$, and $ \delta \in (0,1)$. 
Feature $f$ learned by WSCAT is $\rho_{l_\mathrm{nat}}$-$\gamma_{\Delta}$-robust, where \newline
1) $ \rho_{l_\mathrm{nat}} \le A_1 + l_{\mathrm{m}} \sqrt{\frac{ \log{\frac{1}{\delta}}}{2\vert \mathcal{D}_l \vert}} $ with probability at least $1 - \delta$; \newline
2) $ \gamma_{\Delta} \le \frac{1}{\beta} A_2 + \Delta_{\mathrm{m}} \sqrt{\frac{ \log{\frac{1}{\delta}}}{2\vert \mathcal{D} \vert}} $ with probability at least $1 - \delta$. 
\label{Th-RF}
\end{theorem}
\cref{Th-RF} indicates that: 1) WSCAT leads to smaller $\rho_{l_\mathrm{nat}}$ and $\gamma_{\Delta}$, i.e., more robust feature $f$; 2) the more the semi-supervised data (namely larger $\vert \mathcal{D} \vert$), the smaller the generalization error bounds. \cref{Th-RF} explains how WSCAT utilizes semi-supervised data to learn robust features, which forms the theoretical foundation of WSCAT. 

The detailed proofs of the theorems can be found in supplementary material.

\begin{table}[t]
	\centering
	\caption{Configuration of Datasets. }
	\label{Tab-datasets}
	\setlength{\tabcolsep}{0.15cm}
	\begin{tabular}{l|c|c|c|c}
	\toprule
	\multirow{3}{*}{Datasets}&\multicolumn{2}{c|}{training}&\multirow{3}{*}{testing}&\multirow{3}{*}{classes}\\
	\cmidrule{2-3}
	&{labeled}&{unlabeled}&\\
	\midrule  
	{CIFAR10}&{4,000}&{46,000}&{10,000}&{10}\\
	{CIFAR100}&{10,000}&{40,000}&{10,000}&{100}\\
	{ImageNet32-100}&{10,000}&{116,100}&{5,000}&{100}\\
	\bottomrule
	\end{tabular}
\end{table}

\begin{table*}[t]
	\renewcommand{\arraystretch}{}
	\centering
	\caption{Performance on \textbf{CIFAR10}. The best results of AT methods are bold. }
	\label{Tab-RQ1-CIFAR10}
	\setlength{\tabcolsep}{0.1cm}
	\begin{tabular}{l|c|cccc|c|ccc}
	\toprule
	{Methods}&{Standard}&{TRADES}&{RLFAT}&{DMAT}&{WSCAT-sup}&{MT}&{RST}&{PUAT}&{WSCAT}\\
	\midrule
	{Natural}&{${78.85}_{\pm0.20}$}&{${58.95}_{\pm2.35}$}&{${64.95}_{\pm0.58}$}&{${61.05}_{\pm0.20}$}
	&{${62.82}_{\pm0.19}$}
	&{${87.97}_{\pm0.36}$}&{${73.22}_{\pm0.71}$}&{$\bm{81.88}_{\pm0.16}$}&{${80.93}_{\pm0.14}$}\\
	{FGSM}&{${4.37}_{\pm0.41}$}&{${30.15}_{\pm1.23}$}&{${32.58}_{\pm1.10}$}&{${36.50}_{\pm0.57}$}
	&{${36.12}_{\pm0.16}$}
	&{${43.56}_{\pm1.02}$}&{${49.39}_{\pm0.33}$}&{${54.40}_{\pm0.05}$}&{$\bm{59.62}_{\pm0.16}$}\\
	{PGD}&{${0.00}_{\pm0.00}$}&{${28.09}_{\pm1.15}$}&{${29.52}_{\pm0.95}$}&{${34.75}_{\pm0.49}$}
	&{${34.53}_{\pm0.08}$}
	&{${1.00}_{\pm0.12}$}&{${46.10}_{\pm0.29}$}&{${46.63}_{\pm0.52}$}&{$\bm{58.52}_{\pm0.22}$}\\
	{CW}&{${0.00}_{\pm0.00}$}&{${26.15}_{\pm1.43}$}&{${26.63}_{\pm0.80}$}&{${30.11}_{\pm0.41}$}
	&{${30.60}_{\pm0.21}$}
	&{${0.78}_{\pm0.12}$}&{${43.95}_{\pm0.30}$}&{${46.70}_{\pm0.37}$}&{$\bm{53.15}_{\pm0.08}$}\\
	{AA}&{${0.00}_{\pm0.00}$}&{${25.44}_{\pm1.34}$}&{${25.72}_{\pm0.79}$}&{${29.68}_{\pm0.41}$}
	&{${30.00}_{\pm0.21}$}
	&{${0.00}_{\pm0.00}$}&{${42.94}_{\pm0.24}$}&{${44.46}_{\pm0.28}$}&{$\bm{52.23}_{\pm0.06}$}\\
	\midrule
	{Mean}&{${0.00}_{\pm0.00}$}&{${30.62}_{\pm1.38}$}&{${31.97}_{\pm0.94}$}&{${35.85}_{\pm0.44}$}
	&{${36.08}_{\pm0.18}$}
	&{${0.00}_{\pm0.00}$}&{${49.21}_{\pm0.27}$}&{${52.12}_{\pm0.32}$}&{$\bm{59.40}_{\pm0.05}$}\\
	\bottomrule
	\end{tabular}
\end{table*}

\begin{table*}[t]
	\renewcommand{\arraystretch}{}
	\centering
	\caption{Performance on \textbf{CIFAR100}. The best results of AT methods are bold. }
	\label{Tab-RQ1-CIFAR100}
	\setlength{\tabcolsep}{0.1cm}
	\begin{tabular}{l|c|cccc|c|ccc}
	\toprule
	{Methods}&{Standard}&{TRADES}&{RLFAT}&{DMAT}&{WSCAT-sup}&{MT}&{RST}&{PUAT}&{WSCAT}\\
	\midrule
	{Natural}&{${58.42}_{\pm0.15}$}&{${39.17}_{\pm0.38}$}&{${41.84}_{\pm0.85}$}    &{${37.25}_{\pm0.76}$}
	&{${39.12}_{\pm0.79}$}
	&{${66.76}_{\pm0.40}$}&{${47.54}_{\pm0.15}$}&{${50.21}_{\pm0.22}$}&{$\bm{55.14}_{\pm0.52}$}\\
	{FGSM}&{${3.39}_{\pm0.19}$}&{${14.64}_{\pm0.12}$}&{${15.66}_{\pm0.32}$}&{${10.03}_{\pm0.22}$}
	&{${17.06}_{\pm0.43}$}
	&{${13.06}_{\pm0.29}$}&{${21.59}_{\pm0.07}$}&{${23.66}_{\pm0.39}$} &{$\bm{28.41}_{\pm0.09}$}\\
	{PGD}&{${0.02}_{\pm0.00}$}&{${12.79}_{\pm0.12}$}&{${13.36}_{\pm0.23}$}&{${8.67}_{\pm0.21}$}
	&{${15.43}_{\pm0.47}$}
	&{${0.09}_{\pm0.02}$}&{${18.39}_{\pm0.08}$}&{${19.43}_{\pm0.36}$}&{$\bm{25.26}_{\pm0.32}$} \\
	{CW}&{${0.01}_{\pm0.00}$}&{${12.00}_{\pm0.05}$}&{${11.88}_{\pm0.19}$}&{${5.47}_{\pm0.06}$}
	&{${13.86}_{\pm0.38}$}
	&{${0.08}_{\pm0.02}$}&{${17.58}_{\pm0.05}$}&{${19.18}_{\pm0.29}$}&{$\bm{22.99}_{\pm0.41}$} \\
	{AA}&{${0.00}_{\pm0.00}$}&{${11.31}_{\pm0.17}$}&{${11.06}_{\pm0.10}$}&{${4.88}_{\pm0.09}$}
	&{${13.08}_{\pm0.38}$}
	&{${0.01}_{\pm0.00}$}&{${16.56}_{\pm0.12}$}&{${17.41}_{\pm0.24}$}&{$\bm{21.87}_{\pm0.40}$}  \\
	\midrule
	{Mean}&{${0.00}_{\pm0.00}$}&{${14.54}_{\pm0.11}$}&{${14.83}_{\pm0.22}$}&{${7.94}_{\pm0.15}$}
	&{${16.80}_{\pm0.46}$}
	&{${0.03}_{\pm0.02}$}&{${20.92}_{\pm0.08}$}&{${22.40}_{\pm0.31}$} &{$\bm{27.43}_{\pm0.29}$} \\
	\bottomrule
	\end{tabular}
\end{table*}

\section{Experiment}

The goal of experiments is to answer the following research questions (RQs): 
1) How does WSCAT perform as compared to the state-of-the-art baselines in terms of the standard generalization and adversarial robustness? 2) Can WSCAT conduct complete perturbations? 3) Which components of WSCAT contribute to its performance? 4) How does weakly supervised dynamic loss and the unlabeled samples influence the performance of WSCAT? 5) How long is the training time of WSCAT? All the experiments are conducted on an NVIDIA RTX 4090 GPU. 

\subsection{Experimental Setting}
\subsubsection{Datasets and Baselines}
The experiments are conducted on CIFAR10 \cite{krizhevsky2009learning}, CIFAR100 \cite{krizhevsky2009learning} and ImageNet32 \cite{chrabaszcz2017downsampled}, which are widely used benchmarks for evaluating AT. For ImageNet32, we use the subset consisting of its first 100 classes, and denotes the subset as ImageNet32-100. To simulate the semi-supervised scenario, we partially discard the labels to get unlabeled data on all datasets. The dataset configurations are shown in the \cref{Tab-datasets}, which will be randomly conducted for 3 times. Besides, we randomly select 20\% of the labeled training data as validation set for the tuning of hyper-parameters on all datasets. We compare WSCAT with baseline methods including: 1) supervised non-defensive method \textbf{Standard} \cite{zagoruyko2016wide}; 2) supervised defensive methods \textbf{TRADES} \cite{zhang2019theoretically}, \textbf{RLFAT} \cite{song2019robust}, \textbf{DMAT} \cite{wang2023better} and \textbf{WSCAT-sup} (a variant of WSCAT that only uses labeled data); 3) semi-supervised non-defensive method \textbf{MT} \cite{tarvainen2017mean}; 4) semi-supervised defensive method \textbf{RST} \cite{carmon2019unlabeled} and \textbf{PUAT} \cite{zhang2024provable}.

\subsubsection{Evaluation Protocol}
The performance of our method and the baselines is evaluated in terms of the standard generalization and adversarial robustness. The standard generalization (natural accuracy) is measured by the accuracy on natural test set, which is referred to as \textbf{Natural}, while the adversarial robustness (robust accuracy) is measured by the accuracy on test set after adversarial attacks, i.e., accuracy on AEs generated by attacks over test set. In particular, we choose widely adopted $l_\infty$ norm attacks including \textbf{FGSM} \cite{goodfellow2014explaining}, \textbf{PGD} \cite{madry2017towards}, \textbf{CW} \cite{carlini2017towards} and \textbf{Auto Attack (AA)} \cite{croce2020reliable}. The perturbation budget $\epsilon$ is set to 8/255 for all attacks. We set step size as 1/255, and the number of update steps as 20 for PGD and CW, and use the standard version Auto Attack. Since robust features are predictive for both natural and adversarial data, the harmonic mean of the natural and robust accuracies (i.e., Natural, FGSM, PGD, CW, and AA) referred to as \textbf{Mean} are reported, which is sensitive to smaller extremes and thus can implicitly reflect how robust the learned features are. The robust accuracy against \textbf{NRF} \cite{kim2021distilling}, an attack targeting non-robust features, can be found in supplementary material. 

\begin{table*}[t]
	\renewcommand{\arraystretch}{}
	\centering
	\caption{Performance on \textbf{ImageNet32-100}. The best results of AT methods are bold. }
	\label{Tab-RQ1-ImageNet32}
	\setlength{\tabcolsep}{0.20cm}
	\begin{tabular}{l|c|ccc|c|ccc}
	\toprule
	{Methods}&{Standard}&{TRADES}&{RLFAT}&{WSCAT-sup}&{MT}&{RST}&{PUAT}&{WSCAT}\\
	\midrule
	{Natural}&{${37.53}_{\pm0.42}$}&{${10.08}_{\pm0.46}$}&{${18.37}_{\pm0.25}$} 
	&{${15.75}_{\pm0.01}$}
	&{${45.45}_{\pm2.67}$}&{${28.64}_{\pm1.12}$}&{${25.63}_{\pm1.41}$}&{$\bm{34.64}_{\pm2.76}$}\\
	{FGSM}&{${0.29}_{\pm0.02}$}&{${2.58}_{\pm0.06}$}&{${2.64}_{\pm0.17}$}
	&{${3.06}_{\pm0.24}$}
	&{${1.99}_{\pm0.96}$}&{${6.97}_{\pm0.87}$}&{${8.14}_{\pm0.18}$}&{$\bm{12.63}_{\pm0.13}$}\\
	{PGD}&{${0.00}_{\pm0.00}$}&{${2.45}_{\pm0.05}$}&{${1.96}_{\pm0.09}$}
	&{${2.53}_{\pm0.19}$}
	&{${0.00}_{\pm0.00}$}&{${5.28}_{\pm0.76}$}&{${6.30}_{\pm0.20}$}&{$\bm{9.89}_{\pm0.35}$}\\
	{CW}&{${0.00}_{\pm0.00}$} &{${1.67}_{\pm0.05}$}&{${1.53}_{\pm0.12}$}
	&{${1.94}_{\pm0.18}$}
	&{${0.00}_{\pm0.00}$}&{${4.25}_{\pm0.23}$}&{${6.06}_{\pm0.20}$}&{$\bm{8.01}_{\pm0.37}$}\\
	{AA}&{${0.00}_{\pm0.00}$}&{${1.40}_{\pm0.12}$}&{${1.35}_{\pm0.12}$}
	&{${1.69}_{\pm0.13}$}
	&{${0.00}_{\pm0.00}$}&{${3.68}_{\pm0.32}$}&{${5.31}_{\pm0.03}$}&{$\bm{7.27}_{\pm0.33}$}\\
	\midrule
	{Mean}&{${0.00}_{\pm0.00}$}&{${2.26}_{\pm0.06}$}&{${2.14}_{\pm0.14}$}
	&{${2.64}_{\pm0.21}$}
	&{${0.00}_{\pm0.00}$}&{${5.71}_{\pm0.54}$}&{${7.42}_{\pm0.13}$}&{$\bm{10.59}_{\pm0.32}$} \\
	\bottomrule
	\end{tabular}
\end{table*}

\subsubsection{Hyper-parameters and Implementation}
We employ WRN-28-10 \cite{zagoruyko2016wide} as the target model $C$, and adopt validation-based early stopping \cite{rice2020overfitting} for all methods over all datasets. The optimizer is SGD with Nesterov momentum \cite{nesterov1983method} and cosine learning rate schedule \cite{smith2019super}, where learning rate, weight decay, and momentum are set to 0.1, 5e-4, and 0.9, respectively. One batch consists 128 samples, where the proportion of labeled samples is consistent with that in whole dataset. Meanwhile, we compute the sum in \cref{Eq-CON} over the current batch instead of entire dataset for each update of $C$. A 10-step PGD with $l_\infty$ normed perturbation bound of $8/255$ and step size of $2/255$ is used to generate AEs for all datasets. $\beta$ is tuned on the validation set according to the harmonic mean of natural accuracy and robust accuracy against PGD. On CIFAR10, CIFAR100 and ImageNet32-100, $\beta$ is set to 0.05, 0.05 and 0.2, and $\lambda$ is set to 5.0, 1.0 and 1.0, respectively. 

\subsection{Experimental Results}
\subsubsection{Performance Comparison (RQ1)}

\cref{Tab-RQ1-CIFAR10,Tab-RQ1-CIFAR100,Tab-RQ1-ImageNet32} show the performances of WSCAT and baselines. We can make the following observations: 1) WSCAT surpasses baselines on all datasets in robust accuracies, and at the same time, achieves a high natural accuracy compared with other AT methods, and leads to a highest harmonic mean accuracy. This observation shows that WSCAT can increase adversarial robustness with less sacrifice of standard generalization due to its ability to distill the robust features. 2) Semi-supervised methods always outperform supervised methods, and WSCAT surpasses its supervised variant WSCAT-sup, which justifies our motivation to utilize unlabeled data for better learning of robust features. 3) WSCAT-sup also achieve comparable performance within supervised methods in terms of harmonic mean, which further demonstrates the benefits of complete AEs on breaking the correlations between non-robust features and labels. We also experiment WSCAT-sup under fully-supervised setting with other model architectures, the results of which are reported in supplementary material. 

\subsubsection{Complete Perturbation (RQ2)}

Since the more non-robust features are perturbed, the lower similarity between embeddings of natural samples and their corresponding AEs, we show our complete AE generation strategy's ability to perturb non-robust features as many as possible by this similarity. Specifically, given the same natural dataset, we compare the distribution of similarity between the embeddings of natural samples and AEs generated by TRADES ($\max_{x^\prime \in \mathcal{B}_{\epsilon} (x)} \mathrm{KL} (C(x) \Vert C(x^\prime))$, which is also adopted by DMAT and RST), RLFAT (adds Random Block Shuffle on AE generated by TRADES) and WSCAT ($\max_{x^\prime \in \mathcal{B}_{\epsilon} (x)} \{ \mathrm{KL} (C(x) \Vert C(x^\prime)) + \beta l_\mathrm{con} (z^\prime, z) \}$). From \cref{Fig-similarity}, we can observe that AEs generated by WSCAT are more dissimilar to natural samples than those generated by TRADES or RLFAT, which shows more non-robust features are perturbed by WSCAT due to its complete AE generation fulfilled by the weakly supervised dynamic loss $l_\mathrm{con} (z^\prime, z)$, and empirically demonstrates the conclusion of \cref{Th-CB}. 

\begin{figure}[t]
\centering
\subfloat[CIFAR10]{\includegraphics[width=.3333\columnwidth]{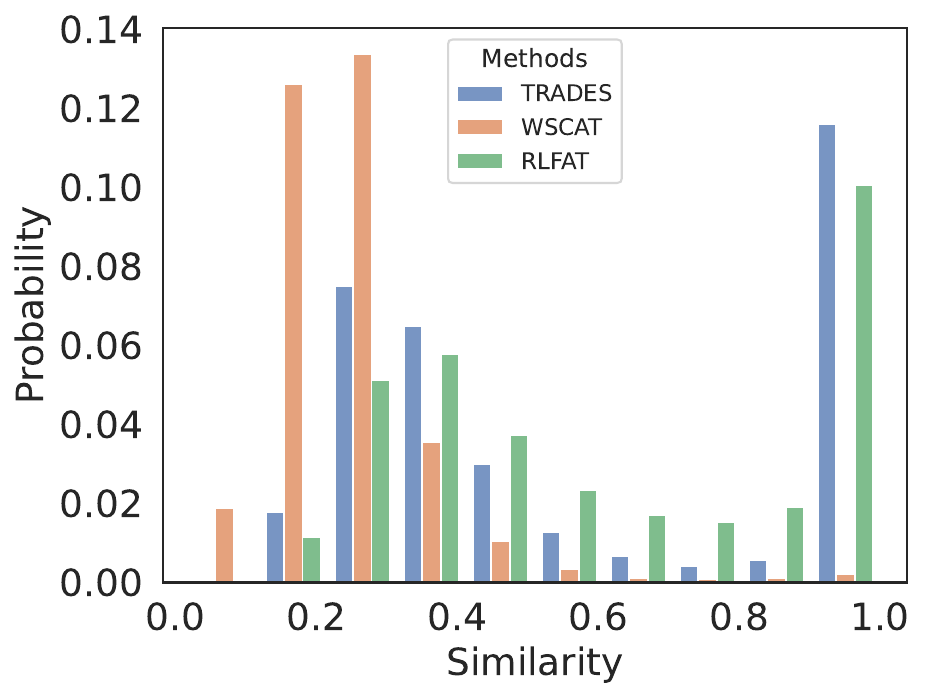} \label{fig-sim-cifar10}} 
\subfloat[CIFAR100]{\includegraphics[width=.3333\columnwidth]{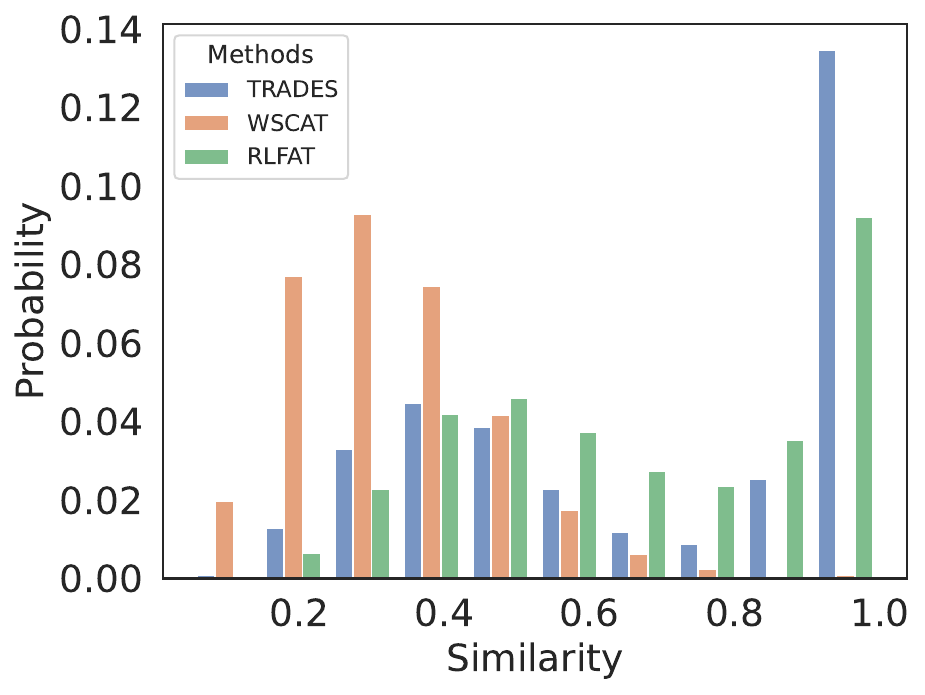} \label{fig-sim-cifar100}}
\subfloat[ImageNet32-100]{\includegraphics[width=.3333\columnwidth]{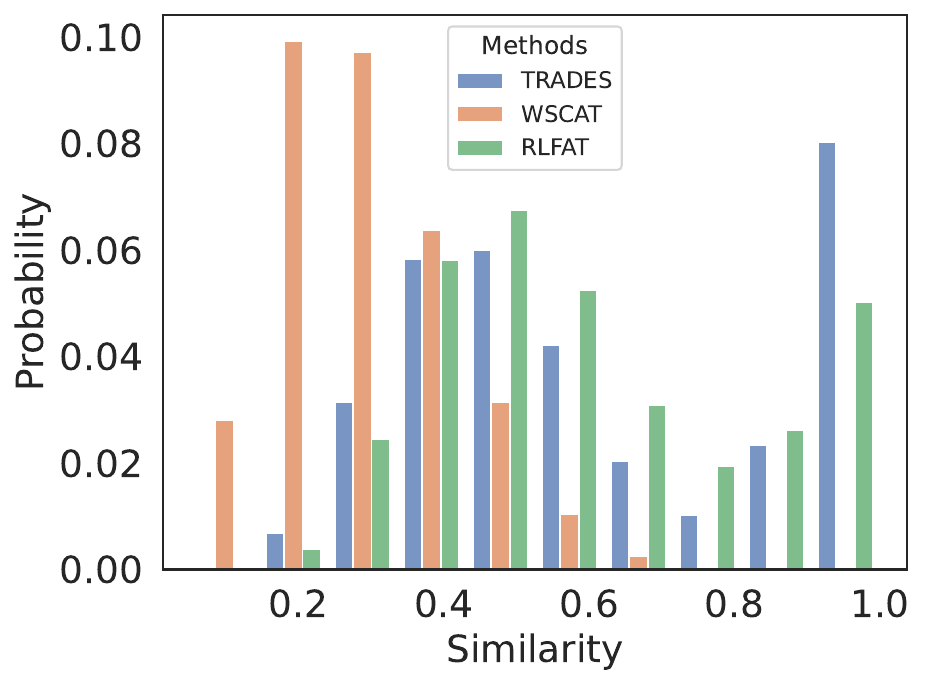} \label{fig-sim-imagenet}}
\caption{Distribution of similarity between the embeddings of natural data and AEs generated by different AT methods. All the AT methods share the same natural dataset. }
\label{Fig-similarity}
\end{figure}

\subsubsection{Ablation Study (RQ3)}

\begin{table}[t]
	\renewcommand{\arraystretch}{}
	\centering
	\caption{The harmonic mean accuracies of WSCAT and its variants. Each harmonic mean is computed over the natural accuracy and the accuracies under FGSM, PGD, CW and AA. }
	\label{Tab-mode}
	\setlength{\tabcolsep}{0.12cm}
	\begin{tabular}{l|c|c|c}
	\toprule
	{Variants}&{CIFAR10}&{CIFAR100}&{ImageNet32-100}\\
	\midrule
	{WSCAT}&{${59.40}_{\pm0.05}$}&{${27.43}_{\pm0.29}$}&{${10.59}_{\pm0.32}$}\\
	{WSCAT-fixed}&{${57.20}_{\pm0.06}$}&{${26.97}_{\pm0.02}$}&{${10.52}_{\pm0.22}$}\\
	{WSCAT-self}&{${57.80}_{\pm0.27}$}&{${26.36}_{\pm0.44}$}&{${10.46}_{\pm0.20}$}\\
	{WSCAT-std}&{${54.88}_{\pm0.03}$}&{${23.83}_{\pm0.05}$}&{${7.09}_{\pm0.27}$}\\
	\bottomrule
	\end{tabular}
\end{table}

\begin{figure}[t]
\centering
\subfloat[WSCAT]{\includegraphics[scale=0.195]{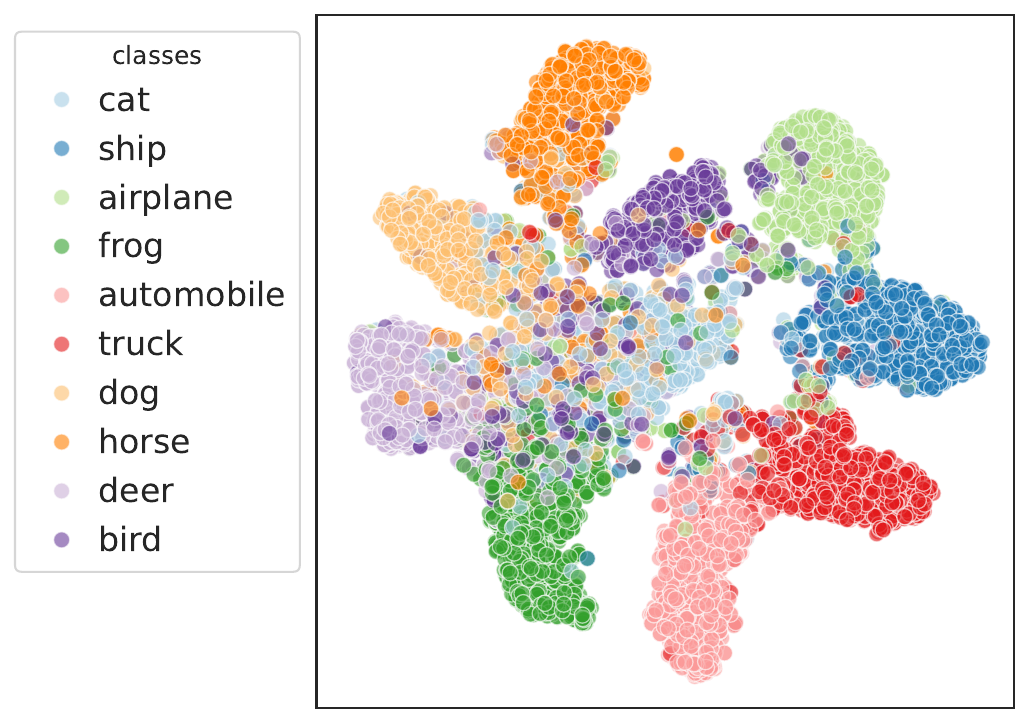} \label{fig-WSCAT}} 
\subfloat[WSCAT-fixed]{\includegraphics[scale=0.195]{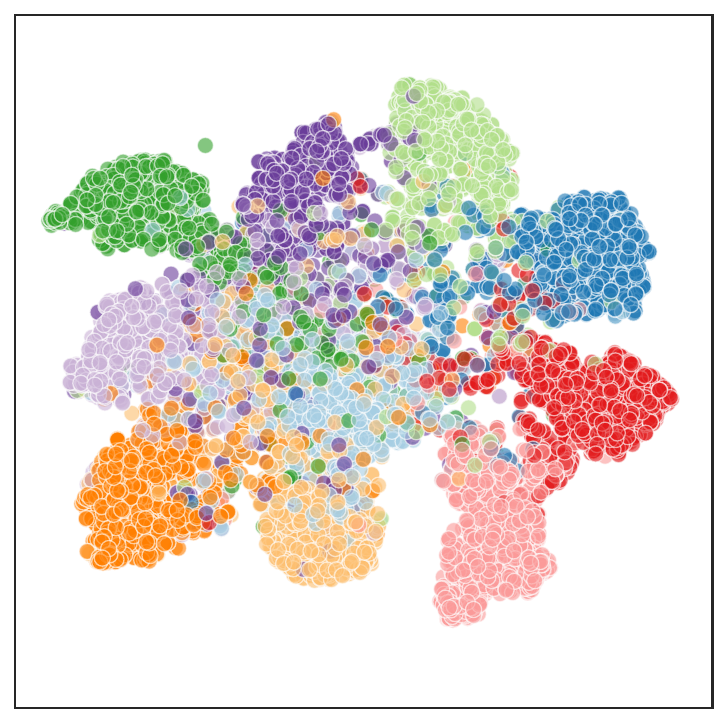} \label{fig-fixed}}
\subfloat[WSCAT-self]{\includegraphics[scale=0.195]{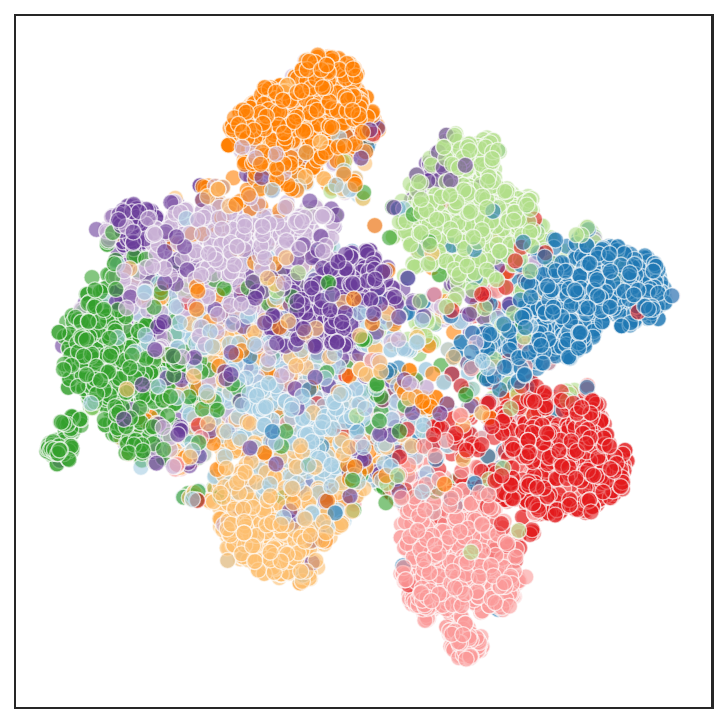} \label{fig-self}}
\caption{Visualization of the embeddings of training samples produced by WSCAT and its variants on CIFAR10. The ground truth labels of the samples are color-coded. }
\label{Fig-feature}
\end{figure}
To further show the superiority of the weakly supervised dynamic loss in \cref{Eq-CON}, we compare WSCAT with two variants WSCAT-fixed (replaces \cref{Eq-CON} by SupCon loss \cite{khosla2020supervised} on pseudo-labeled dataset $\mathcal{D}^*$), and WSCAT-self (replaces \cref{Eq-CON} by Info-NCE loss in \cref{Eq-NCE}), the results of which are reported in \cref{Tab-mode}. The results indicate that WSCAT is superior to the variants since it incorporates the predictions produced by the target model (i.e., weak supervisions), which can dynamically reflect the correlations captured by current model and is independent of pseudo-labels in $\mathcal{D}^*$ that may be noisy, thereby enabling efficient correlation bloacking. We also apply t-SNE algorithm to visualize the feature embeddings on CIFAR10 in \cref{Fig-feature}, from which we can observe that the embeddings generated by WSCAT renders the clusters with higher purity and clearer boundaries than those generated by the two variants (\cref{fig-fixed,fig-self}). 

Additionally, we compare WSCAT with WSCAT-std, which is the variant that generating pseudo-labels by standard training as RST does. Form \cref{Tab-mode}, we can observe that using semi-supervised algorithm MT to produce pseudo-labels of relatively higher quality (than standard training) helps to better learning the robust features, which is because the high-quality pseudo labels help to distill useful features. And compared to the performances of RST in \cref{Tab-RQ1-CIFAR10,Tab-RQ1-CIFAR100,Tab-RQ1-ImageNet32}, one can also see that WSCAT-std beats RST due to the advantages of the complete AE generation. The detailed experiment results are in the supplementary material. 

\subsubsection{Sensitivity Analysis (RQ4)}
We analyze the sensitivity of WSCAT to the weight of the weakly supervised dynamic loss, i.e., $\beta$ in \cref{Eq-AE}, as well as the amount of unlabeled data, while disregarding $\lambda$ since it has been extensively studied in prior works such as \cite{zhang2019theoretically, carmon2019unlabeled}. 

\begin{figure}[t]
\centering
\subfloat[CIFAR10]{\includegraphics[width=.3333\columnwidth]{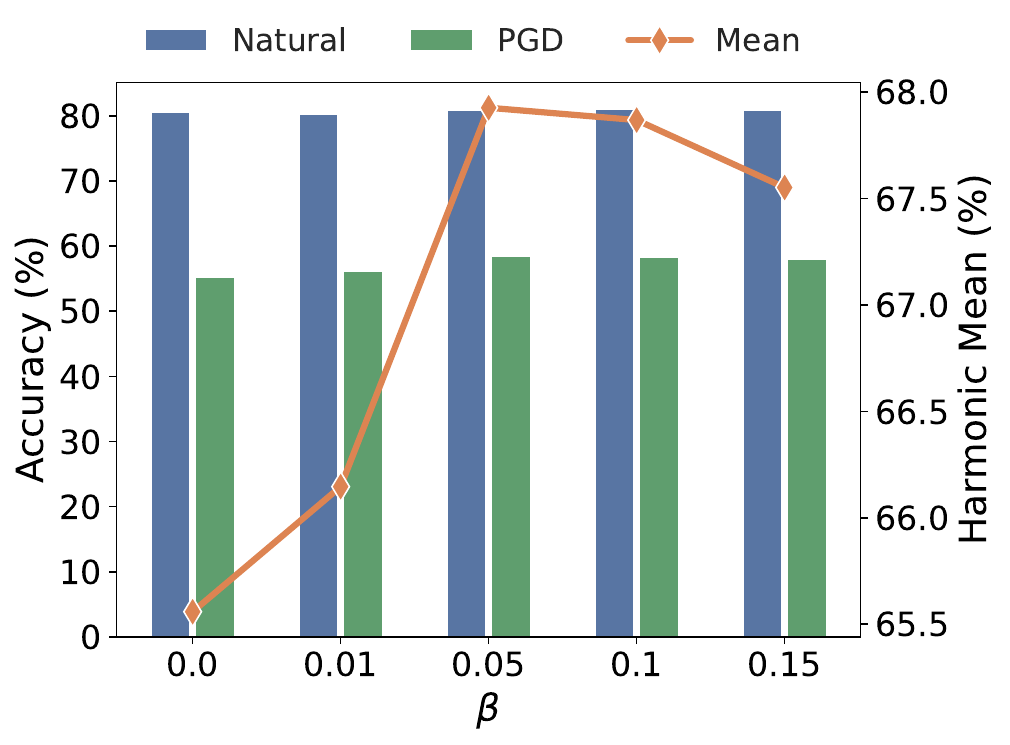} \label{fig-beta-cifar10}} 
\subfloat[CIFAR100]{\includegraphics[width=.3333\columnwidth]{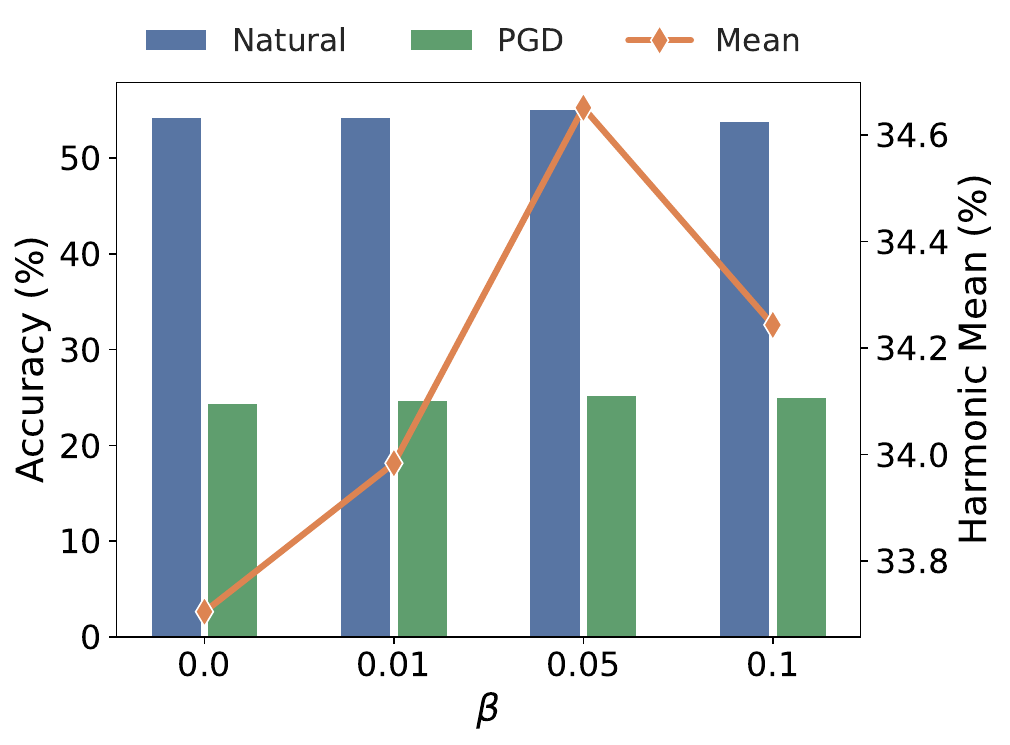} \label{fig-beta-cifar100}}
\subfloat[ImageNet32-100]{\includegraphics[width=.3333\columnwidth]{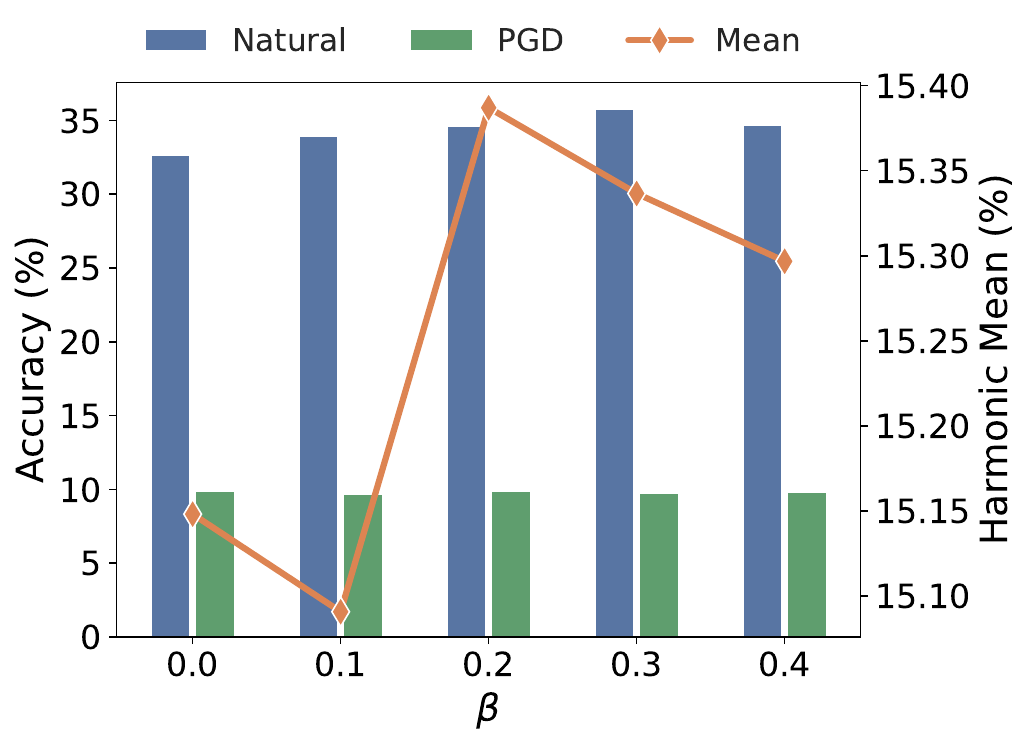} \label{fig-beta-imagenet}}
\caption{Effect of $\beta$. The natural accuracy, robust accuracy against PGD attack are reported with the harmonic mean of them. }
\label{Fig-beta}
\end{figure}

\begin{figure}[t]
\centering
\subfloat[CIFAR10]{\includegraphics[width=.3333\columnwidth]{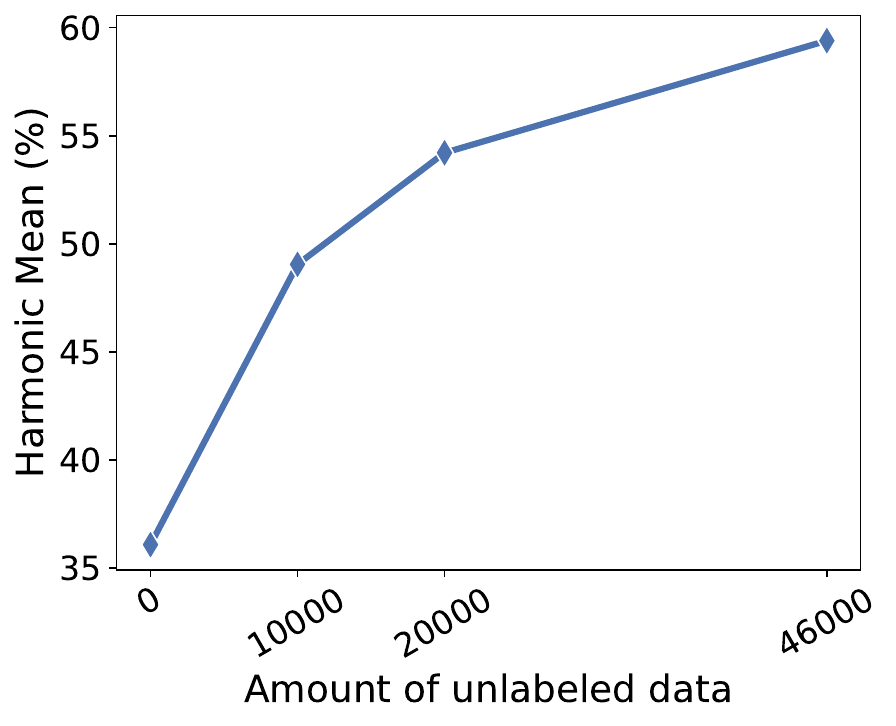} \label{fig-unlabel-cifar10}} 
\subfloat[CIFAR100]{\includegraphics[width=.3333\columnwidth]{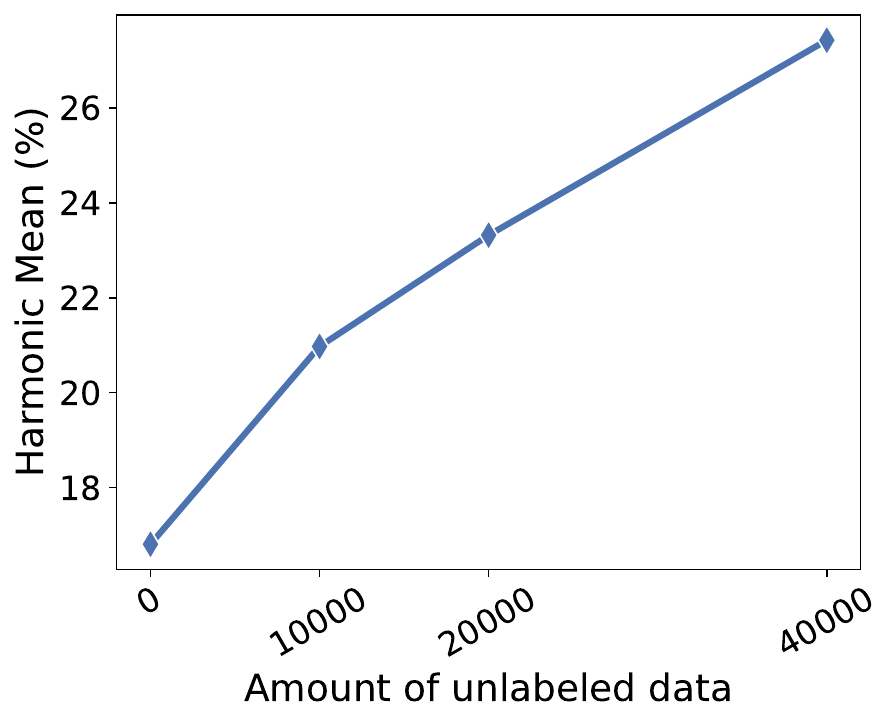} \label{fig-unlabel-cifar100}}
\subfloat[ImageNet32-100]{\includegraphics[width=.3333\columnwidth]{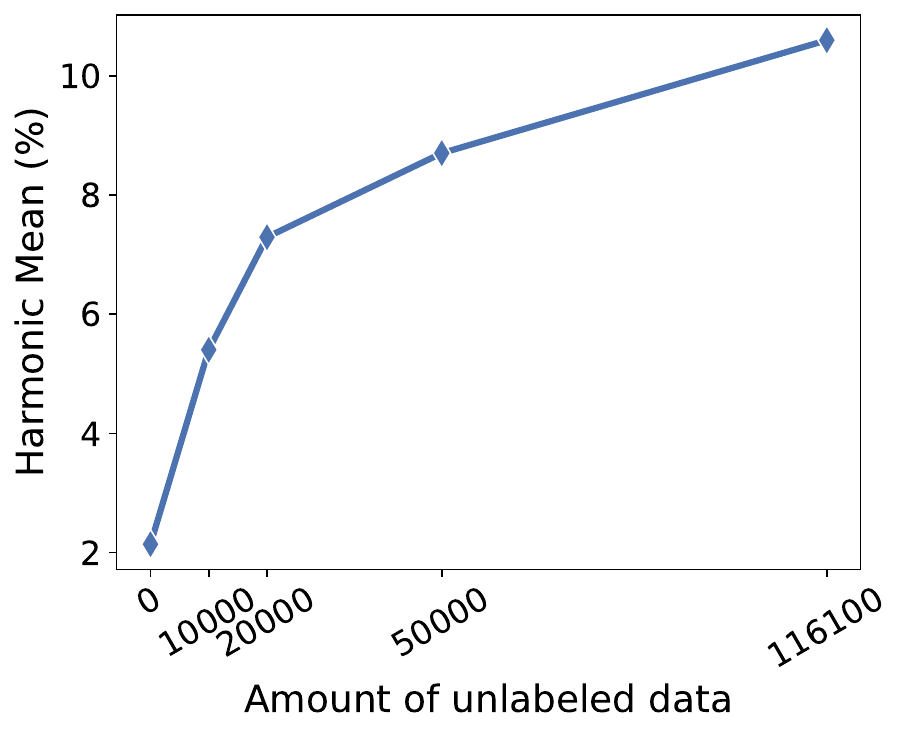} \label{fig-unlabel-imagenet}}
\caption{Effect of the amount of unlabeled data. Each curve is the harmonic mean over natural accuracy and the accuracies under FGSM, PGD, CW and AA. }
\label{Fig-unlabel}
\end{figure}

We first examine the impact of $\beta$ on WSCAT, which also provides insights into the methodology for tuning $\beta$. We run WSCAT with different values of $\beta$ and present the results in \cref{Fig-beta}. As shown, the harmonic mean of natural and robust accuracies initially increases with $\beta$, reaching a peak before subsequently declining. We select the optimal $\beta$ based on this peak value observed on the validation set, which aligns with our theoretical analysis in \cref{sec:moti}. A conservative $\beta$ leaves more non-robust features unperturbed, leading to suboptimal performance. As $\beta$ increases, more non-robust features are allowed to be effectively removed, thereby enabling the target model to learn more robust features. However, an excessively large $\beta$ introduces unintended perturbation of robust features, resulting from that minimizing $\mathrm{MI} (Z^\prime; Z)$ becomes the dominant factor in AE generation. 

Besides, we train WSCAT with different scales of unlabeled data to empirically study influences of unlabeled data, results of which are shown in \cref{Fig-unlabel}. Basically, we see that the more unlabeled data used for training, the better the harmonic mean accuracy. This observation indicates that unlabeled data can indeed help WSCAT learn robust features better, which is consistent with \cref{Th-RF}.

\subsubsection{Training Time (RQ5)}
The epoch times of WSCAT on CIFAR10, CIFAR100 and ImageNet32-100 are $5^{\prime}15^{\prime\prime}$, $5^{\prime}18^{\prime\prime}$ and $13^{\prime}02^{\prime\prime}$, respectively, which are affordable comparing to existing semi-supervised AT (see supplementary material for detail).


\section{Related Work}
\subsection{Semi-supervised Adversarial Training}
Most existing AT methods are supervised and thus suffer from the sparsity of labeled data. Recent research has shown that incorporating a larger pool of unlabeled data can significantly enhance AT performance. For instance, \cite{zhang2019defense} argues that AT performance is impaired when perturbations are generated solely based on labeled data, as they fail to capture the full underlying data distribution. To address this, they propose a semi-supervised AT framework that leverages generative adversarial networks (GANs) to generate AEs while preserving the intrinsic structure of unlabeled data. Inspired by self-training, \cite{carmon2019unlabeled, alayrac2019labels} introduce methods that first assigns pseudo-labels to unlabeled data and then extends AT to a semi-supervised setting using these pseudo-labeled samples. Similarly, \cite{zhang2024provable} jointly trains the target model alongside a G-C-D GAN (for AE generation) on semi-supervised data, improving robustness against various adversarial attacks while mitigating the degradation of standard generalization. However, existing semi-supervised AT methods primarily focus on utilizing unlabeled data to improve overall adversarial robustness but do not explicitly explore how it can be leveraged for learning robust features.


\subsection{Robust Feature Learning}
Since deep learning models have been shown to be vulnerable to adversarial attacks, extensive research has focused on understanding adversarial robustness from the perspective of robust and non-robust features. Recent studies have sought to enhance robustness by redesigning the objectives of traditional AT using various techniques, such as feature disentanglement \cite{yang2021adversarial, yang2021class, wang2024exploring}, multi-view alignment \cite{yang2021structure, zhang2021robust, yu2023improving}, knowledge distillation \cite{kuang2024improving}, and causal inference \cite{zhang2021causaladv, kim2023demystifying}. However, these approaches implicitly assume that AEs are generated by fully perturbing non-robust features. For instance, \cite{zhang2021causaladv} relies on conventional AE generation techniques to modify style variables. In contrast, our work challenges this assumption, arguing that conventional AE generation suffer from incomplete perturbation.

Another line of research focuses on modifying the AE generation. For example, \cite{song2019robust} augments AEs by applying Random Block Shuffle. \cite{xu2022infoat} promotes the generation of more out-of-distribution AEs via point-wise mutual information. Inspired by \cite{kim2021distilling}, which proposes a method to distill robust and non-robust features, \cite{lee2023robust} obtains class-specific robust features using robust proxies. However, despite these advancements, we highlight that the challenge of achieving complete perturbation remains unresolved, as evidenced by the persistent gap between natural and robust accuracy.

\section{Conclusion}
We identify the issue of incomplete perturbation, i.e., not all non-robust features are perturbed in AEs generated by existing adversarial training methods, which hinders the learning of robust features and hence leads to suboptimal adversarial robustness of target classifiers. To tackle this issue, we propose a novel solution called Weakly Supervised Contrastive Adversarial Training (WSCAT), which conducts the learning of robust features by blocking the correlations from non-robust features to labels, via complete AE generation over partially labeled data fulfilled by a novel weakly supervised dynamic loss. The solid theoretical analysis and the extensive experiments demonstrate the superiority of  WSCAT. 


\clearpage
{
    \small
    \bibliographystyle{ieeenat_fullname}
    \bibliography{main}
}

\clearpage
\appendix
\maketitlesupplementary

\section{Proofs of Theorems}\label{app:proof}
\subsection{Useful Lemmas}

\begin{lemma}\label{Le-metric}
Let $z = f(x)$, $z_1 = f(x_1)$ and $\bar{s} (z) = - \frac{1}{\vert \mathcal{N}^+_x \vert } $ $\sum_{x_p \in \mathcal{N}^+_x} \log \frac{\exp s(z, z_p) }{ \sum_{x_n \in \mathcal{D}} \exp{ s(z, z_n)} }$. If $\Delta ( z, z_1 )$$ = $$\vert \bar{s} (z) - \bar{s} (z_1) \vert $ is valid when $x_1 \in \mathcal{B}_{\epsilon}(x)$, $ \Delta( \cdot, \cdot )$ is a distance metric. 
\end{lemma}

\begin{proof}[Proof of \cref{Le-metric}]
It is obvious that for any $x$, $ \Delta( z, z) = 0 $. And for any $x$ and $x_1 \in \mathcal{B}_\epsilon (x)$ the symmetry and non-negativity are clearly satisfied by $\Delta (z, z_1)$. We only need to justify that $ \Delta( \cdot, \cdot )$ satisfies the triangle inequality. For any $x$, $x_1$ and $x_2$, since we have
\begin{equation*}
\begin{aligned}
& \Delta (z, z_1) + \Delta (z, z_2) \\
= & \Delta (z_1, z_2) + \Delta (z, z_2) - \Delta (z_1, z_2) + \Delta (z, z_1)\\
= & \vert \bar{s} (z_1) - \bar{s} (z_2) \vert + \vert \bar{s} (z) - \bar{s} (z_2) \vert - \Delta (z_1, z_2) + \Delta (z, z_1)\\
\ge & \vert \bar{s} (z_1) - 2 \bar{s} (z_2) + \bar{s} (z) \vert - \Delta (z_1, z_2) + \Delta (z, z_1)\\ 
= &  \vert  2\bar{s} (z_1) - 2 \bar{s} (z_2) + \bar{s} (z) - \bar{s} (z_1) \vert - \Delta (z_1, z_2)+\Delta (z, z_1)\\ 
\ge & \vert  2\bar{s} (z_1) - 2 \bar{s} (z_2)\vert - \vert \bar{s} (z) - \bar{s} (z_1) \vert - \Delta (z_1, z_2)+\Delta (z, z_1) \\
= & 2 \Delta(z_1, z_2) - \Delta (z, z_1) - \Delta (z_1, z_2)+\Delta (z, z_1) \\
= &\Delta (z_1, z_2), 
\end{aligned}
\end{equation*}
i.e, the triangle inequality $\Delta (z, z_1) + \Delta (z, z_2) \ge \Delta (z_1, z_2)$ always holds. Therefore, $ \Delta( \cdot, \cdot )$ is a distance metric. 
\end{proof}

\begin{lemma}\label{Le-metric2}
$\vert \l_\mathrm{con}(z^\prime, z) - l_\mathrm{con}(z, z) \vert = \Delta( z^\prime, z )$
\end{lemma}
\begin{proof}[Proof of \cref{Le-metric2}]
According to \cref{Eq-CON}, 
\begin{equation*}
\begin{aligned}
& \l_\mathrm{con}(z^\prime, z) - l_\mathrm{con}(z, z) \\ 
= & - \frac{1}{\vert \mathcal{N}^+_x \vert } \sum_{x_p \in \mathcal{N}^+_x} \log{ \frac{\exp{( s(z^\prime, z_p))} }{ \sum_{x_n \in \mathcal{D}^* } \exp{(s(z^\prime, z_n))} } }\\
& + \frac{1}{\vert \mathcal{N}^+_x \vert } \sum_{x_p \in \mathcal{N}^+_x} \log{ \frac{\exp{( s(z, z_p))} }{ \sum_{x_n \in \mathcal{D}^* } \exp{(s(z, z_n))} } }  \\
= & \bar{s}(z^\prime) - \bar{s}(z). 
\end{aligned}
\end{equation*}
Therefore, $\vert \l_\mathrm{con}(z, z) - l_\mathrm{con}(z^\prime, z) \vert =  \Delta( z^\prime, z )$. 
\end{proof}

\subsection{Detailed Proofs}

\begin{proof}[Proof of \cref{Th-CB}]
By the LogSumExp operation, i.e., $\log (e^{x_1} + e^{x_2} + ... + e^{x_n})$ $ \approx$ $ \max$ $ \{x_1, x_2, ..., x_n\} $, we can transform the contrastive loss in \cref{Eq-CON} to 
\begin{equation*} 
\begin{aligned} 
& l_\mathrm{con} (z^\prime, z) \\
= &  \frac{1}{ \vert \mathcal{N}^+_x \vert } \sum_{x_p \in \mathcal{N}^+_x} \log{ \frac{ \sum_{x_n \in \mathcal{D}^* } e^{s(z^\prime, z_n)} }{e^{ s(z^\prime, z_p)} } }\\
= &  \frac{1}{ \vert \mathcal{N}^+_x \vert } \sum_{x_p \in \mathcal{N}^+_x} \log \Big( e^0 + \sum_{x_n \in \mathcal{D}, x_n \ne x_p} e^{s(z^\prime, z_n) - s(z^\prime, z_p)} \Big)\\
\approx &  \frac{1}{ \vert \mathcal{N}^+_x \vert } \sum_{x_p \in \mathcal{N}^+_x}  \max \big\{ 0, \big\{ s(z^\prime, z_n) - s(z^\prime, z_p) \big\}_{x_n \in \mathcal{D}, x_n \ne x_p} \big\}, 
\end{aligned}
\end{equation*}
from which we can see that maximizing $l_\mathrm{con} (z^\prime, z)$ is approximately maximizing the last line. 
\end{proof}

\begin{proof}[Proof of \cref{Th-RF}]\label{proof-th-rf}

1) According to Hoeffding's Inequality \cite{hoeffding1994probability}, with probability at least $1-\delta$ the following inequality holds: 
\begin{equation*}
\begin{aligned}
& \mathrm{E}_{P_{X, Y}} [ l_\mathrm{nat} (X, Y) ] - \frac{1}{\vert \mathcal{D}^* \vert} \sum_{(x, y)  \in \mathcal{D}^*} l_\mathrm{nat} (x, y) \\
\le & \mathrm{E}_{P_{X, Y}} [ l_\mathrm{nat} (X, Y) ] - \frac{1}{\vert \mathcal{D}_l \vert} \sum_{(x, y)  \in \mathcal{D}_l } l_\mathrm{nat} (x, y) \\
\le & l_{\mathrm{m}} \sqrt{\frac{ \log{\frac{1}{\delta}}}{2\vert \mathcal{D}_l \vert}},
\end{aligned}
\end{equation*}
where the second line holds because existing works \cite{wei2020theoretical, carmon2019unlabeled} have theoretically proven that pseudo-labeled data generated by self-training can decrease the generalization gap. 

Then according to the \cref{Def_robust_feature}, 
\begin{equation*}
\begin{aligned}
\rho_{l_\mathrm{nat}} & = \inf_{g} \mathrm{E}_{P_{X, Y}} [ l_\mathrm{nat} (X, Y)]  \\
& \le A_1 +  l_{\mathrm{m}} \sqrt{\frac{ \log{\frac{1}{\delta}}}{2\vert \mathcal{D}_l \vert}}
\end{aligned}
\end{equation*}
with probability at least $1-\delta$. 

2) Again according to Hoeffding's Inequality, with probability at least $1-\delta$ the following inequality holds: 
\begin{equation*}
\begin{aligned}
\mathrm{E}_{P_{X}} [ \Delta (f(X^\prime), f(X)) ] - \frac{1}{n} \sum_{x \in \mathcal{D}} \Delta (x^\prime, x) 
\le \Delta_{\mathrm{m}}\sqrt{\frac{ \log{\frac{1}{\delta}}}{2\vert \mathcal{D} \vert}}, 
\end{aligned}
\end{equation*}
where $\Delta_{\mathrm{m}}$ is the supremum of the distance $\Delta (\cdot, \cdot)$ over $\{ (z^\prime, z) \vert x \sim P_{X} \land x^\prime \in \mathcal{B}_\epsilon (x) \land z = f(x) \land z^\prime = f(x^\prime) \}$. And thus we have the following inequalities: 
\begin{equation*}
\begin{aligned}
& \mathrm{E}_{P_{X}} [ \sup_{X^\prime \in \mathcal{B}_\epsilon(X) }  \Delta (f(X^\prime), f(X))] - \Delta_{\mathrm{m}} \sqrt{\frac{ \log{\frac{1}{\delta}}}{2\vert D \vert}} \\
\le & \frac{1}{\vert D \vert} \sum_{x \in \mathcal{D}} \sup_{x^\prime \in \mathcal{B}_\epsilon(x) } \Delta (z^\prime, z) \\
\le & \frac{1}{\vert D \vert} \sum_{x \in \mathcal{D}} \sup_{x^\prime \in \mathcal{B}_\epsilon(x) } \big\{ \l_\mathrm{con}(z^\prime, z) - l_\mathrm{con}(z, z) \big\} \\
\le & \frac{1}{\beta\vert D \vert} \sum_{x \in \mathcal{D}} \sup_{x^\prime \in \mathcal{B}_\epsilon(x) } \big\{ \mathrm{KL} (C(x) \Vert C(x^\prime)) \\
& + \beta ( l_\mathrm{con}(z^\prime, z) - \l_\mathrm{con}(z, z) ) \big\} \\
= & \frac{1}{\beta\vert D \vert} \sum_{x \in \mathcal{D}} \sup_{x^\prime \in \mathcal{B}_\epsilon(x) }  l_\mathrm{adv} (x^\prime, x) \\
= &  \frac{2}{\beta} A_2, 
\end{aligned}
\end{equation*}
where the first inequality holds with probability at least $1-\delta$. Then according to \cref{Def_robust_feature}, we can get that 
\begin{equation*}
\begin{aligned}
\gamma_{\Delta} &  = \mathrm{E}_{ P_{X}} [ \sup_{X^\prime \in \mathcal{B}_{\epsilon}(X)} \Delta (f(X^\prime), f(X))]  \\
& \le \frac{2}{\beta} A_2 + \Delta_{\mathrm{m}} \sqrt{\frac{ \log{\frac{1}{\delta}}}{2\vert D \vert}}.
\end{aligned}
\end{equation*}

Therefore, feature $f$ captured by the target model $C = g \circ f$ trained by WSCAT is $\rho_{l_\mathrm{nat}}$-$\gamma_{\Delta}$-robust, where $\rho_{l_\mathrm{nat}} \le l_{\mathrm{m}} \sqrt{\frac{ \log{\frac{1}{\delta}}}{2\vert \mathcal{D}_l \vert}}$ with probability at least $1 - \delta$, and $\gamma_{\Delta} \le \frac{2}{\beta} A_2 + \Delta_{\mathrm{m}} \sqrt{\frac{ \log{\frac{1}{\delta}}}{2\vert D \vert}}$ with probability at least $1 - \delta$. 
\end{proof}

\section{Additional Experimental Results} 
\subsection{Performance Comparison (RQ1)}\label{supp:per}
The performance of WSCAT-sup and TRADES under fully-supervised setting across various model architectures is shown in \cref{tab-sup}. 

\begin{table}[t]
	\renewcommand{\arraystretch}{}
	\centering
	\caption{Performance of models trained by Standard, TRADES and WSCAT-sup (a variant of our WSCAT that uses only labeled data) under fully-supervised setting. }
	\label{tab-sup}
	\setlength{\tabcolsep}{0.06cm}
	\resizebox{0.48\textwidth}{!}{
	\begin{tabular}{cc|cccccc|c}
	\toprule
	\makecell[c]{Dataset\\(Model)}&{Method}&{Nat.}&{FGSM}&{PGD}&{CW}&{AA}&{Mean}&{NRF}\\
	\midrule
	\multirow{3}{*}{\makecell[c]{CIFAR10\\(ResNet50)}}&{Standard}&{${95.15}$}&{${42.37}$}&{${0.02}$}&{${0.01}$}&{${0.00}$}&{${0.00}$}&{${0.00}$} \\
	&{TRADES}&{${80.34}$}&{${56.05}$}&{${51.74}$}&{${49.27}$}&{${47.99}$}&{${55.10}$}&{${43.92}$} \\
	&{WSCAT-sup}&{${82.37}$}&{${59.84}$}&{${57.84}$}&{${52.50}$}&{${51.41}$}&{${59.07}$}&{${45.86}$} \\
	\midrule
	\multirow{3}{*}{\makecell[c]{CIFAR10\\(ResNet152)}}&{Standard}&{${95.26}$}&{${49.42}$}&{${0.01}$}&{${0.00}$}&{${0.00}$}&{${0.00}$}&{${0.00}$} \\
	&{TRADES}&{${81.52}$}&{${56.56}$}&{${51.55}$}&{${49.96}$}&{${48.15}$}&{${55.48}$}&{${57.01}$} \\
	&{WSCAT-sup}&{${80.98}$}&{${59.92}$}&{${58.46}$}&{${52.89}$}&{${52.09}$}&{${59.35}$}&{${58.56}$} \\
	\midrule
	\multirow{3}{*}{\makecell[c]{CIFAR10\\(WRN28-10)}}&{Standard}&{${96.23}$}&{${43.34}$}&{${0.01}$}&{${0.02}$}&{${0.00}$}&{${0.00}$}&{${0.00}$} \\
	&{TRADES}&{${84.65}$}&{${60.92}$}&{${56.34}$}&{${54.12}$}&{${52.85}$}&{${59.97}$}&{${50.44}$} \\
	&{WSCAT-sup}&{${84.18}$}&{${61.42}$}&{${59.72}$}&{${54.02}$}&{${52.91}$}&{${60.74}$}&{${55.12}$} \\
	\midrule
	\multirow{3}{*}{\makecell[c]{CIFAR100\\(WRN28-10)}}&{Standard}&{${78.62}$}&{${16.27}$}&{${0.34}$}&{${0.09}$}&{${0.00}$}&{${0.00}$}&{-} \\
	&{TRADES}&{${58.69}$}&{${33.74}$}&{${30.77}$}&{${28.31}$}&{${27.02}$}&{${33.00}$}&{-} \\
	&{WSCAT-sup}&{${59.71}$}&{${34.61}$}&{${32.63}$}&{${28.80}$}&{${27.46}$}&{${33.92}$}&{-} \\
	\bottomrule
	\end{tabular}
	}
\end{table}

\subsection{Ablation Study (RQ3)}\label{supp:abl}
The performance of WSCAT and WSCAT's different variants is shown in \cref{Tab-ap-CIFAR10,Tab-ap-CIFAR100,Tab-ap-imagenet}. 

\begin{table}[t]
	\renewcommand{\arraystretch}{}
	\centering
	\caption{Performance of different variants on \textbf{CIFAR10}. }
	\label{Tab-ap-CIFAR10}
	\setlength{\tabcolsep}{0.1cm}
	\resizebox{0.48\textwidth}{!}{
	\begin{tabular}{lcccc}
	\toprule
	{Methods}&{WSCAT}&{WSCAT-fixed}&{WSCAT-self}&{WSCAT-std}\\
	\midrule
	{Natural}&{${80.93}_{\pm0.14}$}&{${79.04}_{\pm0.45}$}&{${80.72}_{\pm0.12}$}&{${76.65}_{\pm0.30}$}\\
	{FGSM}&{${59.62}_{\pm0.16}$}&{${57.56}_{\pm0.22}$}&{${58.71}_{\pm0.25}$}&{${55.33}_{\pm0.37}$}\\
	{PGD}&{${58.52}_{\pm0.22}$}&{${54.55}_{\pm0.17}$}&{${54.58}_{\pm0.43}$}&{${53.75}_{\pm0.18}$}\\
	{CW}&{${53.15}_{\pm0.08}$}&{${51.66}_{\pm0.03}$}&{${52.20}_{\pm0.34}$}&{${48.68}_{\pm0.10}$}\\
	{AA}&{${52.23}_{\pm0.06}$}&{${50.77}_{\pm0.06}$}&{${51.20}_{\pm0.34}$}&{${48.00}_{\pm0.02}$}\\
	\midrule
	{Mean}&{${59.40}_{\pm0.05}$}&{${57.20}_{\pm0.06}$}&{${57.80}_{\pm0.27}$}&{${54.88}_{\pm0.03}$}\\
	\bottomrule
	\end{tabular}}
\end{table}

\begin{table}[t]
	\renewcommand{\arraystretch}{}
	\centering
	\caption{Performance of different variants on \textbf{CIFAR100}. }
	\label{Tab-ap-CIFAR100}
	\setlength{\tabcolsep}{0.1cm}
	\resizebox{0.48\textwidth}{!}{
	\begin{tabular}{lcccc}
	\toprule
	{Methods}&{WSCAT}&{WSCAT-fixed}&{WSCAT-self}&{WSCAT-std}\\
	\midrule
	{Natural}&{${55.14}_{\pm0.52}$}&{${55.09}_{\pm0.08}$}&{${54.70}_{\pm1.48}$}&{${51.66}_{\pm0.18}$}\\
	{FGSM}&{${28.41}_{\pm0.09}$}&{${27.43}_{\pm0.35}$}&{${27.55}_{\pm0.49}$}&{${25.22}_{\pm0.46}$}\\
	{PGD}&{${25.26}_{\pm0.32}$}&{${23.84}_{\pm0.36}$}&{${24.08}_{\pm0.11}$}&{${21.89}_{\pm0.29}$}\\
	{CW}&{${22.99}_{\pm0.41}$}&{${22.65}_{\pm0.03}$}&{${22.04}_{\pm0.58}$}&{${19.39}_{\pm0.40}$}\\
	{AA}&{${21.82}_{\pm0.40}$}&{${21.83}_{\pm0.01}$}&{${20.77}_{\pm0.72}$}&{${18.70}_{\pm0.15}$}\\
	\midrule
	{Mean}&{${27.43}_{\pm0.29}$}&{${26.97}_{\pm0.02}$}&{${26.36}_{\pm0.44}$}&{${23.83}_{\pm0.05}$}\\
	\bottomrule
	\end{tabular}}
\end{table}

\begin{table}[t]
	\renewcommand{\arraystretch}{}
	\centering
	\caption{Performance of different variants on \textbf{ImageNet32-100}. }
	\label{Tab-ap-imagenet}
	\setlength{\tabcolsep}{0.1cm}
	\resizebox{0.48\textwidth}{!}{
	\begin{tabular}{lcccc}
	\toprule
	{Methods}&{WSCAT}&{WSCAT-fixed}&{WSCAT-self}&{WSCAT-std}\\
	\midrule
	{Natural}&{${34.64}_{\pm2.76}$}&{${33.28}_{\pm0.00}$}&{${32.32}_{\pm0.00}$}&{${31.43}_{\pm0.11}$}\\
	{FGSM}&{${12.63}_{\pm0.13}$}&{${12.54}_{\pm0.00}$}&{${12.62}_{\pm0.00}$}&{${8.61}_{\pm0.19}$}\\
	{PGD}&{${9.89}_{\pm0.35}$}&{${9.94}_{\pm0.00}$}&{${9.80}_{\pm0.00}$}&{${6.90}_{\pm0.20}$}\\
	{CW}&{${8.01}_{\pm0.37}$}&{${8.02}_{\pm0.00}$}&{${8.06}_{\pm0.00}$}&{${5.11}_{\pm0.23}$}\\
	{AA}&{${7.27}_{\pm0.33}$}&{${7.14}_{\pm0.00}$}&{${7.06}_{\pm0.00}$}&{${4.61}_{\pm0.23}$}\\
	\midrule
	{Mean}&{${10.59}_{\pm0.32}$}&{${10.52}_{\pm0.00}$}&{${10.46}_{\pm0.00}$}&{${7.09}_{\pm0.27}$}\\
	\bottomrule
	\end{tabular}}
\end{table}

\subsection{Training Time (RQ5)}\label{supp:time}
\begin{table}[t]
	\renewcommand{\arraystretch}{}
	\centering
	\caption{Epoch time of WSCAT and RST. }
	\label{Tab-ap-time}
	\setlength{\tabcolsep}{0.20cm}
	\begin{tabular}{l|c|c|c}
	\toprule
	{Datasets}&{CIFAR10}&{CIFAR100}&{ImageNet32-100}\\
	\midrule
	{WSCAT}&{$5^{\prime}15^{\prime\prime}$}&{$5^{\prime}18^{\prime\prime}$}&{$13^{\prime}02^{\prime\prime}$}\\
	{RST}&{$5^{\prime}14^{\prime\prime}$}&{$5^{\prime}18^{\prime\prime}$}&{$13^{\prime}20^{\prime\prime}$}\\
	\bottomrule
	\end{tabular}
\end{table}

To show WSCAT does not excessively increases the training time than existing semi-supervised AT methods, we compare WSCAT's epoch time with that of RST, which is an efficient semi-supervised AT method \cite{zhang2024provable}. The result is shown in \cref{Tab-ap-time}, from which one can observe that WSCAT does not bring additional training time cost overall. The result is reasonable since during a batch of the training, the loss defined in \cref{Eq-CON} can be calculated just based on points in that batch instead of the entire dataset.



\end{document}